\documentclass{article}
\usepackage[top=1in, bottom=1in, left=1in, right=1in]{geometry}
\usepackage[utf8]{inputenc}

\usepackage{wrapfig,lipsum}

\usepackage{caption}
\usepackage[numbers]{natbib}
\usepackage{url}
\usepackage{graphicx}
\usepackage{booktabs}
\usepackage{enumitem}
\usepackage{subcaption}
\usepackage{amsmath,amssymb,amsfonts,amsthm}

\usepackage{xcolor}
\usepackage{mathtools}
\usepackage{dsfont}
\usepackage{hyperref}
\usepackage{bm}
\usepackage{nicefrac}
\usepackage{wrapfig}
\usepackage{lipsum}

\newcounter{assumption}\renewcommand{\theassumption}{\arabic{assumption}}

\newenvironment{assumption}[1][]{\begin{trivlist}\item[] \refstepcounter{assumption} {\bf{Assumption\ \theassumption\ }}{(#1)}.  }{ \ifvmode\smallskip\fi\end{trivlist}}

\def\vx{{\bf x}}

\def\vz{{\bf z}}
\def\vv{{\bf v}}

\newcommand*\E[1]{\mathbb{E}\left[#1\right]}

\newcommand*\lrp[1]{\left(#1\right)}
\newcommand*\lrn[1]{\left\|#1\right\|}

\def\vx{{\bf x}}

\def\vz{{\bf z}}
\def\vv{{\bf v}}

\newcommand{\real}{\ensuremath{\mathbb{R}}}

\newcommand{\z}{z}

\newtheorem{lemma}{Lemma}
\newtheorem{theorem}{Theorem}

\usepackage{multirow}
\usepackage{caption}
\usepackage{subcaption}

\newcommand{\ve}[1]{\mathbf{#1}}

\usepackage{algorithm, algcompatible}
\usepackage{algpseudocode}
\algnewcommand\algorithmicinput{\textbf{Input:}}
\algnewcommand\INPUT{\item[\algorithmicinput]}
\algnewcommand\algorithmicoutput{\textbf{Output:}}
\algnewcommand\OUTPUT{\item[\algorithmicoutput]}
\algnewcommand\algorithmicoptional{\textbf{Optional:}}
\algnewcommand\OPTIONAL{\item[\algorithmicoptional]}

\usepackage{amsmath,amsfonts,bm}

\def\eqref#1{equation~\ref{#1}}

\def\1{\bm{1}}

\def\ve{{\bm{e}}}

\def\vp{{\bm{p}}}

\def\vr{{\bm{r}}}

\def\vu{{\bm{u}}}
\def\vv{{\bm{v}}}

\def\vx{{\bm{x}}}

\def\vz{{\bm{z}}}

\def\mK{{\bm{K}}}

\def\mW{{\bm{W}}}
\def\mX{{\bm{X}}}

\def\mZ{{\bm{Z}}}

\DeclareMathAlphabet{\mathsfit}{\encodingdefault}{\sfdefault}{m}{sl}
\SetMathAlphabet{\mathsfit}{bold}{\encodingdefault}{\sfdefault}{bx}{n}

\DeclareMathOperator*{\argmin}{arg\,min}

\newtheorem*{theorem*}{Theorem}

\usepackage{hyperref}
\usepackage{url}
\usepackage{overpic}
\usepackage{booktabs}
\usepackage{tabularx}
\usepackage{thm-restate}

\title{Sparse identification of nonlinear dynamics and Koopman operators with Shallow Recurrent Decoder Networks}

\author{Mars Liyao Gao$^*$, Jan P. Williams$^\dag$ and J. Nathan Kutz$^{\ddag,**}$\\[.1in]
{
$^*$ Computer Science \& Engineering, University of Washington, Seattle, WA} \\
{ $^\dag$ Mechanical Engineering, University of Washington, Seattle, WA }\\
{ $^\ddag$ Applied Mathematics, University of Washington, Seattle, WA  }\\
{$^{**}$ Electrical and Computer Engineering, University of Washington, Seattle, WA }\\[.2in]
\texttt{\{marsgao,jmpw1,kutz\}@uw.edu} \\
}

\date{\today}

\begin{document}

\maketitle

\begin{abstract}
Modeling real-world spatio-temporal data is exceptionally difficult due to inherent high dimensionality, measurement noise, partial observations, and often expensive data collection procedures. 
In this paper, we present \textbf{S}parse \textbf{I}dentification of \textbf{N}onlinear \textbf{Dy}namics with \textbf{SH}allow \textbf{RE}current \textbf{D}ecoder networks (SINDy-SHRED), a method to jointly solve the sensing and model identification problems with simple implementation, efficient computation, and robust performance.
SINDy-SHRED uses Gated Recurrent Units to model the temporal sequence of sparse sensor measurements along with a shallow decoder network to reconstruct the full spatio-temporal field from the latent state space.
Our algorithm introduces a SINDy-based regularization for which the latent space progressively converges to a SINDy-class functional, provided the projection remains within the set.
In restricting SINDy to a linear model,  a Koopman-SHRED model  is generated.
SINDy-SHRED (i) learns a symbolic and interpretable generative model of a parsimonious and low-dimensional latent space for the complex spatio-temporal dynamics, (ii) discovers new physics models even for well-known physical systems, (iii) achieves provably robust convergence with an observed globally convex loss landscape, and (iv)  achieves superior accuracy, data efficiency, and training time, all with fewer model parameters.
We conduct systematic experimental studies on PDE data such as turbulent flows, real-world sensor measurements for sea surface temperature, and direct video data. 
The interpretable SINDy and Koopman models of latent state dynamics enable stable and accurate long-term video predictions, outperforming all current baseline deep learning models in accuracy, training time, and data requirements, including Convolutional LSTM, PredRNN, ResNet, and SimVP.

\end{abstract}

\section{Introduction}

Partial differential equations (PDEs) derived from first principles or qualitative behavior remain the most ubiquitous class of models used to describe physical, spatio-temporal phenomena. However, for complex dynamics it is often the case that the simplifying assumptions necessary to construct a PDE model can render it ineffectual for real data where the physics is high-dimensional, multi-scale in nature, only partially known, or where first principles models currently do not exist. In such cases, machine learning (ML) methods offer an attractive alternative for learning both the physics and coordinates (fundamental variables) necessary to quantify the observed spatiotemporal phenomena.  Many recent efforts utilizing ML techniques seek to relax the computational burden for PDE simulation by learning surrogate models to forward-simulate or predict spatiotemporal systems.  However, this new machine learning paradigm frequently exhibits instabilities during the training process, unstable roll-outs when modeling future state predictions, and often yields minimal computational speedups~\citep{mcgreivy2024weak}.

The {\em Shallow REcurrent Decoder} (SHRED) network~\citep{williams2023sensing} is a recently introduced architecture that utilizes data from sparse sensors to reconstruct and predict the entire spatiotemporal domain. Similar to Takens' embedding theorem, SHRED models trade spatial information at a single time point for a trajectory of sensor measurements across time. Previous work has shown SHRED can achieve excellent performance in sensing and reduced order modeling examples ranging from weather and atmospheric forecasting~\citep{williams2023sensing}, plasma physics~\cite{kutz2024shallow}, nuclear reactors~\cite{riva2024robust}, and turbulent flow reconstructions~\cite{tomasetto2025reduced}. 
Theoretically rooted in the classic PDE method of separation of variables \citep{williams2023sensing, tomasetto2025reduced}, the decoding-only strategy of SHRED circumvents the computation of inverse pairs, i.e. an encoder and the corresponding decoder.  It has been well-known for decades that the computation of the inverse of a matrix is highly unstable and ill-posed~\cite{forsythe1977computer,higham2002accuracy,bao2020regularized}. By decoding only, SHRED avoids this problem and learns a single embedding without the corresponding inversion.
\begin{figure}[t]
    \centering
    \includegraphics[width=\textwidth]{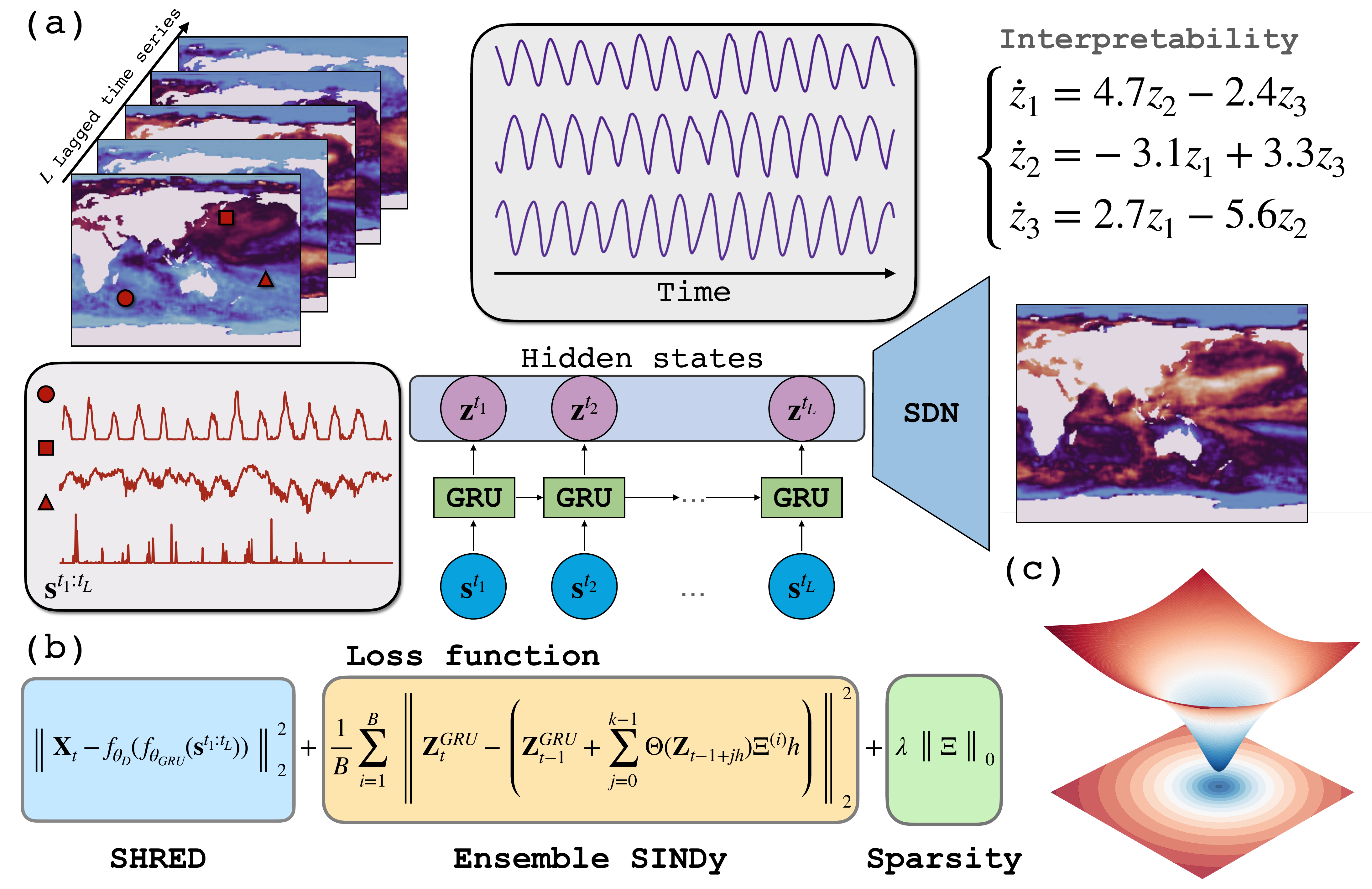}
    \caption{(a) Illustration of the SINDy-SHRED and Koopman-SHRED architecture. SINDy-SHRED transfers the original sparse sensor signal (red) to an interpretable latent representation (purple) that falls into the SINDy-class functional. This framework can be adapted into Koopman-SHRED by restricting the library $\Theta(\cdot)$ to be linear. The shallow decoder performs a reconstruction in the state space. We obtain an interpretable linear model for the sea-surface temperature data considered (details in Sec.~\ref{sec:sst_expr}). (b) The loss function consists of three parts: (i) the SHRED loss controls the reconstruction accuracy of the state space, (ii) the Ensemble SINDy loss helps to model the parsimonious dynamics of the latent space, and (iii) the sparsity constraint identifies the governing equation within this optimization framework. (c) We visualize the globally convex loss landscape of SINDy-SHRED as in~\citep{li2018visualizing}. }
    \label{fig:figure1_pre}
\end{figure}

In this paper, we introduce \textbf{S}parse \textbf{I}dentification of \textbf{N}onlinear \textbf{Dy}namics with \textbf{SH}allow \textbf{RE}current \textbf{D}ecoder networks (SINDy-SHRED).
SINDy-SHRED exploits the latent space of recurrent neural networks for sparse sensor modeling, and enforces interpretability via a SINDy-based functional class. 
{Moreover, our theoretical analysis (see Thm.~\ref{thm:xi_error} and Thm.~\ref{thm:nn_error}) rigorously demonstrates key advantages of latent space modeling with SINDy over conventional neural networks, i.e. under mild conditions it is guaranteed to converge with a bounded error.}
In this way, SINDy-SHRED enables a robust and sample-efficient joint discovery of governing equation and coordinate system. 
With the correct governing equation, SINDy-SHRED can perform an accurate long-term prediction in a learned, low-dimensional latent space, and in turn allows for long-term forecasting in the original spatiotemporal (pixel) space. 
In further restricting SINDy to a linear model, a Koopman approximation~\cite{brunton2021modern} can be constructed to produce a Koopman-SHRED architecture.

SINDy-SHRED and Koopman-SHRED are lightweight models which can perform low-rank recovery with only a few active sensors for the spatio-temporal fields, which is critical for large-scale scientific data modeling and real-time control. 
Specifically, for $D$-dimensional fields, $D+1$ sensors are required for disambiguation of the spatio-temporal field much like localization for cellular networks~\cite{zekavat2019handbook}. 
Moreover, the proposed architecture does not require large amounts of data during training, thereby avoiding a common pitfall in existing ML techniques for accelerating physics simulations and enabling rapid training on a single laptop.
SINDy-SHRED/Koopman-SHRED are also highly reproducible with minimal effort in hyperparameter tuning due to the  globally  convex  optimization landscape, as illustrated in the lower right panel of Fig.~\ref{fig:figure1_pre}. 
The recommended network structure, hyperparameters, and training setting can generalize to many different datasets~\footnote{Code implementation: \url{https://github.com/gaoliyao/sindy-shred}.}\footnote{Run in Google Colab: \url{https://colab.research.google.com/drive/1Xxw3P_x9a8iKZ6RPe2ZfTb8rJoWtPwTK}.}. In short, we demonstrate SINDy-SHRED/Koopman-SHRED to be robust and highly applicable in many modern scientific modeling problems. In what follows, we will refer generally to the proposed architecture as SINDy-SHRED, with Koopman-SHRED as a special case.

We perform a wide range of studies to demonstrate the effectiveness of SINDy-SHRED.
We apply the model on the sea surface temperature data, which is a complex real-world problem.
We also consider data from complex simulations of atmospheric chemistry, 2D Kolmogorov flow, and isotropic turbulence, video data of flow over a cylinder, and video data of a pendulum. The ability of SINDy-SHRED to perform well on video data is an important result for so-called ``GoPro physics,'' whereby physics is learned directly from video recordings. 
With extremely small sample size and noisy environments, SINDy-SHRED achieves governing equation identification with stable long-term predictions, thus overcoming the critical limitations for long-term forecasts of spatio-temporal phenomena that include instabilities and massive computational requirements.
To summarize, the contributions of our paper are the following: 
\begin{itemize}
    \item We propose SINDy-SHRED to learn a symbolic and interpretable generative model of a parsimonious and low-dimensional latent space (fundamental variables) of recurrent neural networks for complex spatio-temporal dynamics.
    \item Even for the well-known physical systems considered, we discover in each case a new and hitherto unknown physics model characterizing the spatio-temporal dynamics directly from measurement or video streams. These benchmark systems include turbulent flows, videos of physical systems, and PDE simulations.
    \item Under weak assumptions, we establish rigorous error bounds on the robust convergence of the SINDy-SHRED algorithm which allows us to guarantee model performance, a rarity in deep learning.  Moreover, we systematically observe that the algorithm has a global convex loss landscape, preventing the need for any hyper-parameter tuning (except for the SINDy parameters of dimension and sparsity).
    \item We compare SINDy-SHRED to state-of-the-art deep learning algorithms in spatio-temporal prediction, showing that the architecture achieves superior accuracy, data efficiency, training time, and long term prediction, all with fewer model parameters.
\end{itemize}
\vspace{-1mm}
\section{Related works}
\vspace{-1mm}

Traditionally, spatio-temporal physical phenomena are modeled by Partial Differential Equations (PDEs). 
To accelerate PDE simulations, recent efforts have leveraged neural networks to model physics. 
By explicitly assuming knowledge of the underlying PDE, physics-informed neural networks~\citep{raissi2019physics} utilize the PDE structure as a constraint for small sample learning. 
However, assuming the exact form of governing PDE for real data can be a strong limitation. 
There have been many recent works on learning and predicting PDEs directly using neural networks~\citep{khoo2021solving,li2020neural,holl2020learning,lu2021learning,lin2021accelerated,long2018pde}. 
Alternatively, variants of SINDy~\citep{rudy2017data,messenger2021weak,fasel2022ensemble} offer a data-driven approach to identify PDEs from the spatial-temporal domain. 
High-dimensionality and data requirements  can be prohibitive for many practical applications. 

In parallel, efforts have been directed toward the discovery of physical laws through dimensionality reduction techniques~\citep{champion2019data,lusch2018deep,mars2024bayesian}, providing yet another perspective on the modeling of scientific data. 
The discovery of physics from a learned latent spaces has previously been explored by~\citep{fukami2021sparse,cheng2024latent,farenga2024latent,conti2023reduced,wu2022learning,li2020visual,ozalp2024stability,otto2019linearly,monsel2024deep,chen2022automated}, yet none of these methods consider a regularization on the latent space with no explicit encoder.
Yu et al. proposed the idea of physics-guided learning~\citep{yu2024learning} which combines physics simulations and neural network approximations. 
Directly modeling physics from video is also the subject of much research in the field of robotics~\citep{finn2016unsupervised,todorov2012mujoco,sanchez2018graph}, computer vision~\citep{xie2024physgaussian,wu2017learning,assran2023self,bardes2024revisiting} and computer graphics~\citep{kandukuri2020learning,liuphysgen,wu2015galileo,mrowca2018flexible}, as these fields also require better physics models for simulation and control. 
From the deep learning side, combining the structure of differential equations into neural networks~\citep{he2016deep,chen2018neural} has been remarkably successful in a wide range of tasks. 
As an example, when spatial-temporal modeling is framed as a video prediction problem, He et al. found~\citep{he2022masked} that random masking can be an efficient spatio-temporal learner, and deep neural networks can provide very good predictions for the next 10 to 20 frames~\citep{shi2015convolutional,wang2017predrnn,gao2022simvp,guen2020disentangling}. 
Generative models have also been found to be useful for scientific data modeling~\citep{mirza2014conditional,song2021maximum,cachay2024probabilistic,solera2024beta,eiximeno2025pylom}.

\vspace{-1mm}
\section{Methods}
\vspace{-1mm}

The shallow recurrent decoder network (\textbf{SHRED}) is a computational technique that utilizes recurrent neural networks to estimate high-dimensional states from limited measurements~\citep{williams2023sensing}. 
The method functions by trading high-fidelity spatial information for trajectories of sparse sensor measurements at given spatial locations. Mathematically, consider a high-dimensional data series $\{\mX_i\}_{i=1}^T\in\real^{(W\times H)\otimes T}$ that represents the evolution of a spatio-temporal dynamical system, where $W$, $H$, and $T$ denote the width, height, and total time steps of the system, respectively.
In SHRED, each sensor collects data from a fixed spatial position in a discretized time domain. Denoting the subset of sensors as $\mathcal{S}$, the input data of SHRED is $\{\mX^{\mathcal{S}}\}_{i=1}^T\in\real^{\text{card}(S)\otimes T}$.
Provided the underlying PDE allows spatial information to propagate, these spatial effects will appear in the time history of the sensor measurements, enabling the sensing of the entire field using only a few sensors.
In vanilla SHRED, a Long Short-Term Memory (LSTM) module is used to map the sparse sensor trajectory data into a latent space, followed by a shallow decoder to reconstruct the entire spatio-temporal domain at the current time step. 

SHRED enables efficient sparse sensing that is widely applicable to many scientific problems~\citep{ebers2024leveraging,kutz2024shallow,riva2024robust,tomasetto2025reduced}.
The advantage of SHRED comes from three aspects. First, SHRED only requires minimal sensor measurements. 
Under practical constraints, collecting full-state measurements for data prediction and control can be prohibitively expensive. 
Second, SHRED does not require grid-like data collection, which allows for generalization to more complex data structures. 
For example, it is easy to apply SHRED to graph data with an unknown underlying structure, such as human motion data on joints, robotic sensor data, and financial market data.
Furthermore, SHRED is grounded in the classical technique of separation of variables~\citep{tomasetto2025reduced}.

\subsection{Empowering SHRED with representation learning and physics discovery}

To achieve a parsimonious representation of physics, it is important to find a representation that effectively captures the underlying dynamics and structure of the system.
In SINDy-SHRED (shown in Fig.~\ref{fig:figure1_pre}), we extend the advantages of SHRED, and perform a joint discovery of coordinate systems and governing equations. 
This is accomplished by enforcing that the latent state of the recurrent network follows an ODE in the SINDy class of functions.

\paragraph{Finding better representations}  SHRED has a natural advantage in modeling latent governing physics due to its small model size. 
SHRED is based on a shallow decoder with a relatively small recurrent network structure. The relative simplicity of the model allows the latent representation to maintain many advantageous properties such as smoothness and Lipschitzness. 
Experimentally, we observe that the hidden state space of a SHRED model is generally very smooth. 
Second, SHRED does not have an explicit encoder, which avoids the potential problem of spectral bias~\citep{rahaman2019spectral}. 
Many reduced-order modeling methods that rely on an encoder architecture struggle to learn physics and instead focus only on modeling the low-frequency information (background)~\citep{refinetti2022dynamics,champion2019data,mars2024bayesian}.
Building upon SHRED, we further incorporate SINDy to regularize the learned recurrence with a well-characterized and simple form of governing equation. 
This approach is inspired by the principle in physics that, under an ideal coordinate system, physical phenomena can be described by a parsimonious dynamics model~\citep{champion2019data,mars2024bayesian}. 
When the latent representation and the governing law are well-aligned, this configuration is likely to capture the true underlying physics. 
This joint discovery results in a latent space that is both interpretable and physically meaningful, enabling robust and stable future prediction based on the learned dynamics.

\subsection{Latent space regularization via SINDy and Koopman}

As a compressive sensing procedure, there exist infinitely many equally valid solutions for the latent representation. 
Therefore, it is not necessary for the latent representation induced by SHRED to follow a well-structured differential equation.
For instance, even if the exhibited dynamics are fundamentally linear, the latent representation may exhibit completely unexplainable dynamics, making the model challenging to interpret and extrapolate.
Therefore, in SINDy-SHRED, our goal is to further constrain the latent representations to lie within the SINDy-class functional.
This regularization promotes models that are fundamentally explainable by a SINDy-based ODE, allowing us to identify a parsimonious governing equation. 
The SINDy class of functions typically consists of a library of commonly used functions, which includes polynomials and Fourier series. 
Although they may seem simple, these functions possess surprisingly strong expressive power, enabling the model to capture very complex dynamical systems.

\subsubsection{SINDy as a Recurrent Neural network}
We first reformulate SINDy using a neural network form, simplifying its incorporation into a SHRED model. 
ResNet~\citep{he2016deep} and Neural ODE~\citep{chen2018neural} utilize skip connections to model residual and temporal derivatives. 
Similarly, this could also be done via a Recurrent Neural Network (RNN) which has a general form of 
\begin{align}
    z_{t+1}=z_{t} + f(x_t),
\end{align}
where $f(\cdot)$ is some function of the input. From the Euler method, the ODE forward simulation via SINDy effectively falls into the category of Recurrent Neural Networks (RNNs) which has the form
\begin{align}
    z_{t+1}=z_{t} + f_{\Theta}(x_t, \Xi, \Delta t),
\end{align}
where $f_\Theta(x_t, \Xi, \Delta t)=\Theta(x_t)\Xi\Delta t$ is a nonlinear function.  Notice that this $f_\Theta(\cdot)$ has exactly the same formulation as in SINDy~\citep{brunton2016discovering}. 
We demonstrate a visual form of this design in the Appendix. 
The application of function libraries with sparsity constraints is a manner of automatic neural architecture search (NAS)~\citep{zoph2016neural}. 
Compared to all prior works~\citep{champion2019data,fukami2021sparse,conti2023reduced}, this implementation of the SINDy unit fits better in the framework of neural network training and gradient descent. 
We utilize trajectory data $\{\vz_i\}_{i=1}^T$ and forward simulate the SINDy-based ODE using a trainable parameter $\Xi$. To achieve better stability and accuracy for forward integration, we use Euler integration with $k$ mini-steps (with time step $\frac{\Delta t}{k}$) to obtain $\vz_{t+1}$. 
In summary, defining $h = \frac{\Delta t}{k}$, we optimize $\Xi$ with the following: 
\begin{align}
\Xi = \arg\min \left\| \mathbf{z}_{t+1} - \left( \mathbf{z}_{t} + \sum_{i=0}^{k-1} \Theta(\mathbf{z}_{t + i h}) \Xi h \right) \right\|_2^2, \quad \mathbf{z}_{t + i h} = \mathbf{z}_{t} + \Theta(\mathbf{z}_{t + (i-1) h}) \Xi h, \quad\min \lrn{\Xi}_0.
\end{align}
To achieve $\ell_0$ optimization, we perform pruning with $\ell_2$ regularization which is known to approximate $\ell_0$ regularization under regularity conditions~\citep{zheng2014high,gao2023convergence,blalock2020state}. Applying SINDy unit has the following benefits:
(a) The SINDy-function library contains frequently used functions in physics modeling (e.g. polynomials and Fourier series).  
(b) With sparse system identification, the neural network is more likely to identify governing physics, which is fundamentally important for extrapolation and long-term stability. 

\subsubsection{Latent space regularization via ensemble SINDy}

We first note that we deviate from the original SHRED architecture by using a GRU as opposed to an LSTM. This choice was made because we generally found that GRU provides a smoother latent space and propagates only a single hidden state.
Now, recall that our goal is to find a SHRED model with a latent state that is within the SINDy-class functional.  However, the initial latent representation from SHRED does not follow the SINDy-based ODE structure at all. 
On the one hand, if we naively apply SINDy to the initial latent representation, the discovery is unlikely to fit the latent representation trajectory.
On the other hand, if we directly replace the GRU unit to SINDy and force the latent space to follow the discovered SINDy model, it might lose information that is important to reconstruction the entire spatial domain. 
Therefore, it is important to let the two latent spaces align progressively. 

In Algorithm~\ref{alg:sindy_regularization}, we describe our training procedure that allows the two trajectories to progressively align with each other. To further ensure a gradual adaptation and avoid over-regularization, we introduce ensemble SINDy units with varying levels of sparsity constraints, which range in effect from promoting a full model (all terms in the library are active) to a null model (where no dynamics are represented). 
From the initial latent representation $\vz^{\text{iter 0}}_{1:t}$ from SHRED, the SINDy model first provides an initial estimate of ensemble SINDy coefficients $\{\hat{\Xi}_0^i\}_{i=b}^B$. 
Then, the parameters of SHRED will be updated towards the dynamics simulated by $\{\hat{\Xi}_0\}_{i=b}^B$, which generates a new latent representation trajectory $\vz^{\text{iter 1}}_{1:t}$. 
We iterate this procedure and jointly optimize the following loss function to let the SHRED latent representation trajectory approximate the SINDy generated trajectory: 
\begin{align}
\mathcal{L}= \left\| \mX_{t} - f_{\theta_{D}}(f_{\theta_{\text{GRU}}}(\mX_{t-L:t}^{\mathcal{S}})) \right\|_2^2 +  \sum_{i=1}^{B} \left\| \mZ^{\text{GRU}}_{t} - \left( \mZ_{t-1}^{\text{GRU}} + \sum_{j=0}^{k-1} \Theta(\mZ_{t-1+jh}) \Xi^{(i)} h \right) \right\|_2^2 + \lambda \left\| \Xi \right\|_0,
\end{align}
where $\mZ_{t -1+ i h} = \mZ_{t} + \Theta(\mZ_{t -1+ (i-1) h}) \Xi h$, $\mZ_{t -1}=\mZ_{t-1}^{\text{GRU}}$, and $h=\frac{\Delta t}{k}$.

\begin{algorithm*}[t]
    \caption{{Latent state space regularization via SINDy}}
    \label{alg:sindy_regularization}
    \begin{algorithmic}[1]
    \INPUT{input $\mX_{t-L:t+1}^\mathcal{S}$, $\mX_{t}$, SINDy library $\Theta(\cdot)$, timestep $\Delta t$.}
    \Function{LatentSpaceSINDy}{$\mX_{t-L:t+1}^\mathcal{S}, \mX_{t+1}, \Delta t$}
        \For{i in $0, 1, \cdots, n-1$:}
            \State $\mZ_t,\;\mZ_{t+1} = f_{\theta_{\text{GRU}}}(\mX_{t-L:t}^\mathcal{S}),\;f_{\theta_{\text{GRU}}}(\mX_{t-L+1:t+1}^\mathcal{S})$;
                                    \For{j in $(0, 1, \frac{\Delta t}{k})$:} \Comment{SINDy forward simulation}
                \State $\mZ^{\text{SINDy}}_{t+\frac{j+1}{k}\Delta t}=\mZ^{\text{SINDy}}_{t+\frac{j}{k}\Delta t}+\Theta(\mZ^{\text{SINDy}}_{t+\frac{j}{k}\Delta t})\Xi \Delta t$
            \EndFor
            \State $\hat{\mX}_{t+1}=f_{\theta_D}(\mZ_{t+1})$\Comment{SHRED reconstruction}
            \State $\theta_{\text{GRU}}, \Xi, \theta_D = \arg\min_{\theta_{\text{GRU}}, \Xi, \theta_D} \left\| \mathbf{X}_{t+1} - \hat\mX_{t+1}\right\|_2^2 + \left\| \mathbf{Z}^{\text{GRU}}_{t+1} - \mZ^{\text{SINDy}}_{t+1} \right\|_2^2 + \lambda \left\| \Xi \right\|_0$
            \If{$i\mod 100 = 0$}
            \State $\Xi[|\Xi|<\text{threshold}]=0$
            \EndIf
        \EndFor \Comment{Train until converges}
    \EndFunction
    \end{algorithmic}
\end{algorithm*}

\subsubsection{Latent space linearization via the Koopman operator} Koopman operator theory~\citep{koopman1931hamiltonian,koopman1932dynamical} provides an alternative approach to solving these problems by linearizing the underlying dynamics. The linearized embedding is theoretically grounded to be able to represent nonlinear dynamics in a linear framework, which is desirable for many applications in science and engineering, including control, robotics, weather modeling, and so on. 
From the transformed measurements $\vz=g(\vx)$ from the true system $\vx$, the Koopman operator $\mathcal{K}$ is an infinite-dimensional linear operator given by
\begin{align}
    \mathcal{K}g:=g\circ \mathbf{F},
\end{align}
where $\mathbf{F}(\cdot)$ describes the law of the dynamical system in its original space that $\vx_{t+1}=\mathbf{F}(\vx_t)$. 
The Koopman operator enables a coordinate transformation from $\vx$ to $\vz$ that linearizes the dynamics
\begin{align}
    \mathcal{K}g(\vx_{t+1})=g(\vx_{t}).
\end{align}

However, it is generally impossible to obtain the exact form of this infinite-dimensional operator, so a typical strategy is to find a finite-dimensional approximation of the Koopman operator by means of data-driven approaches~\citep{brunton2021modern}. 
In Koopman-SHRED, we utilize the GRU unit to approximate the eigenfunctions. In the latent space, we enforce and learn the linear dynamics, represented by a matrix 
$\mK$, which corresponds to the latent space evolution derived from the eigenfunctions.
We keep the interpretability of the model via a parsimonious latent space. 
In implementation, a simple strategy is to adapt the SINDy-unit with only the linear terms. 
This models a continuous version of Koopman generator that $\frac{d}{dt}\z(\vx(t))=\mathcal{G}_t z(\vx(t))$ where the corresponding Koopman operator $\mathcal{K}_t=e^{t\mathcal{G}}$. 
Then, we follow a similar practice in SINDy-SHRED, updating the Koopman-regularized loss function: 
\begin{align}
\mathcal{L}= \left\| \mX_{t} - f_{\theta_{D}}(f_{\theta_{\text{GRU}}}(\mX_{t-L:t}^{\mathcal{S}})) \right\|_2^2 +  \left\| \mZ^{\text{GRU}}_{t+m} - \mathbf{K}^m\mZ^{\text{GRU}}_{t-1}\right\|_2^2.
\end{align}

To enable continuous spectrum for better approximation for the Koopman operator, one could adapt an additional neural network to learn the eigenvalues $\lambda_i$'s and form the linear dynamics $K$ from the learned eigenvalues~\citep{lusch2018deep}. 
We also provide detailed implementation and experimental results in the Appendix. 
This flexibility allows the system to learn dynamical systems in a more general setting, accommodating various initial conditions and experimental setups. 
In our experimental study, we observe that this strategy may perform comparably to learning a fixed $\mK$ for linear dynamics under fixed experimental environments.

\subsection{Theoretical analysis on dynamical system learning}

Suppose that the dynamical system has the form $\dot{\vx}=f(\vx)$, and we have measurements of 
\begin{align}
\vx=\{\vx_1, \vx_2, ...,\vx_t, \vx_{t+1},..., \vx_T\},
\end{align}
with time gap $\Delta t$. 
The empirical loss function $L_n(\cdot)$ of a function $f$ given the empirical distribution $P_n$ which consists of data samples $\mathcal{D}=\{(\vx_i, \dot{\vx_i})\}_{i=1}^n$ is 
\begin{align}
    L_n\lrp{f} = \frac{1}{n}\sum_{i=1}^n \ell\lrp{f(\vx_i), \vx_i},
\end{align}
while the true loss is $L\lrp{f} = \E{\ell\lrp{f(x), y}}$.

From empirical process theory, there is a generalization gap between $L_n(f)$ and $L(f)$, which is correlated with the expressive power of the functional class~\citep{wainwright2019high,van1996weak}. When the expressive power of the estimator functional class is small, the generalization error is expected to be mild, resulting from a smooth adaptation from local perturbations.
On the other hand, when the expressive power of the estimator is excessive, we expect that local perturbations can cause significant shifts in performance, leading to a larger generalization error.

Therefore, even though neural networks are universal approximators~\citep{hornik1989multilayer}, they can perform poorly in extrapolation. 
SINDy effectively avoids this issue because the functional class is moderately sized, and the continuous formulation allows SINDy to reach a much lower error. 
We present the following theorems to compare error bounds for dynamical system learning between the SINDy-class and  the neural network class. First, for a SINDy-class, we have the following theorem. 

\begin{theorem}
    \label{thm:xi_error}
    Suppose we have a SINDy-class functional with a library of functions $\Theta(x)$, and we estimate the coefficient vector $\xi$ by $\hat{\xi}$ using least squares. Let $\xi^*$ be the true coefficient vector. 
    Under regularity conditions in the Appendix, we have the expected error in predicting the dynamical system after time $T$ is bounded by
    \begin{align}
    E||\hat{x}(T) - x(T)|| \leq  \mathcal{O}\left(e^{LT} s \sqrt{\frac{p}{n}}\right)
    \end{align}
    where $L$ is the Lipschitz constant of the system, $s$ is the level of noise, $p$ is the number of functions in the library, and $n$ is the number of samples.

    Furthermore, with probability $1-\delta$, the error in predicting the dynamical system after time $T$ is bounded by
    \begin{align}
    ||\hat{x}(T) - x(T)|| \leq  \mathcal{O}\left(e^{LT} s \sqrt{\frac{p}{n} \log\left(\frac{1}{\delta}\right)}\right).
    \end{align}
\end{theorem}
Similarly, for the neural network class of functions, we have: 
\begin{theorem}
    \label{thm:nn_error}
    Suppose we have a neural network functional with $k$ layers of ReLU activation functions and parameters $\theta=\lrp{\mW_1, \dots, \mW_k}$, which computes functions
    \begin{align}
        f(\vx;\theta)=\sigma_k\lrp{\mW_k\sigma_{k-1}\lrp{\mW_{k-1}\cdots\sigma_1\lrp{\mW_1 \vx}}}.
    \end{align}
    Assume that the weights are bounded such that $\lrn{\mW_i}_F \leq B$ for all $i$.
        Then, with probability at least $1-\delta$, the error in predicting the dynamical system after $H = T/\Delta t$ steps is bounded by
    \begin{align}
        ||\hat{x}(T) - x(T)|| \leq  \mathcal{O}\left((\log n)^4 B^{k(H+1)} \sqrt{\frac{k}{n}} + \frac{\log(1/\delta)}{n}\right),
    \end{align}
    where $n$ is the number of samples.
\end{theorem}

From the error bounds above, we notice that compared to traditional learning problems (e.g. classification, regression), dynamical system learning is more challenging as the error grows exponentially with increasing time horizon.
Short-term local error accumulates as the differential equation steps forward, leading to a large error for longer-time extrapolation. 
This exponential error behavior exists even for linear problems.
From the error bound, it becomes evident why neural networks may encounter significant challenges in learning such problems. 
The challenge is mainly due to the presence of the term $B^{kH}$.
The bound on weights $B$ is expected to increase significantly for complex dynamical systems that need to model large local temporal derivatives. 
And, the number of layers $k$ is expected to increase for more complex and chaotic dynamics. 
To generate high-quality video sequences, finer temporal steps are required for a higher frame-per-second rate, which further exacerbates error growth. 
Therefore, the $B^{kH}$ term could be extremely large in practice.
 
To control these errors contributed by $B^{kH}$, one needs to collect exponentially more samples to ensure the generalization error is small. 
This explains the observation that many video generation methods relying purely on neural networks struggle to follow physical laws accurately, as the required sample size is prohibitively large.
Transformer networks tend to memorize example trajectories in their key-value pairs, but its mechanism is not inherently designed to capture the exact physical law, which can often produce poor extrapolations from unseen initial conditions~\citep{zeng2023transformers}.

\vspace{-1mm}
\section{Computational Experiments}
\vspace{-1mm}
In the following, we perform case studies across a range of scientific and engineering problems.  We begin with building models directly from video data, which includes flow around a cylinder and a pendulum.  We then consider real-world sea-surface temperature data and conclude with turbulent flow examples.

\subsection{GoPro physics video data: flow over a cylinder}

In this subsection, we demonstrate the performance of SINDy-SHRED on an example of so-called ``GoPro physics modeling.''
The considered data is collected from a dyed water channel to visualize a flow over a cylinder~\citep{albright2023flow}. The Reynolds number is 171 in the experiment. 
The dataset contains 11 seconds of video taken at 30 frames per second (FPS). 
We transfer the original RGB channel to gray scale and remove the background by subtracting the mean of all frames. 
After the prior processing step, the video data has only one channel (gray) within the range $(0, 1)$ with a height of 400 pixels and a width of 1,000 pixels. 
We randomly select and fix 200 pixels as sensor measurements from the entire $400,000$ space, which is equivalent to only $0.05 \%$ of the data. We set the lag parameter to 60 frames. 
We include the details of the experimental settings of SINDy-SHRED in the Appendix~\ref{app:pen_expr_detail}.

\noindent \textbf{SINDy-SHRED discovery} We define the representation of the hidden latent state space as $(z_1, z_2, z_3, z_4)$. 
We discover the following dynamical system: 
\begin{align}
\label{eqn:flow_equation}
\begin{cases}
    \dot{z}_1 & = -0.69 z_2 + 0.98 z_3 -0.40 z_4, \\
\dot{z}_2 & = 1.00 z_1 -0.78 z_3 1 -0.31 z_2 z_3^2, \\
\dot{z}_3 & = -1.029 z_1 + 0.59 z_2 + 0.41 z_4. \\
\dot{z}_4 & = -0.26 z_1^2 -0.29 z_2^2 z_3 -0.39 z_3^3.
\end{cases}
\end{align}
The identified nolinear system has two fixed points, $\mathbf{z}=0$ and $\mathbf{z} = \begin{pmatrix}
    -0.11 & -0.23 & -0.14 & 0.05
\end{pmatrix}^T$, both of which are unstable.

\begin{figure}[t]
    \centering
    \begin{overpic}[width=\textwidth]{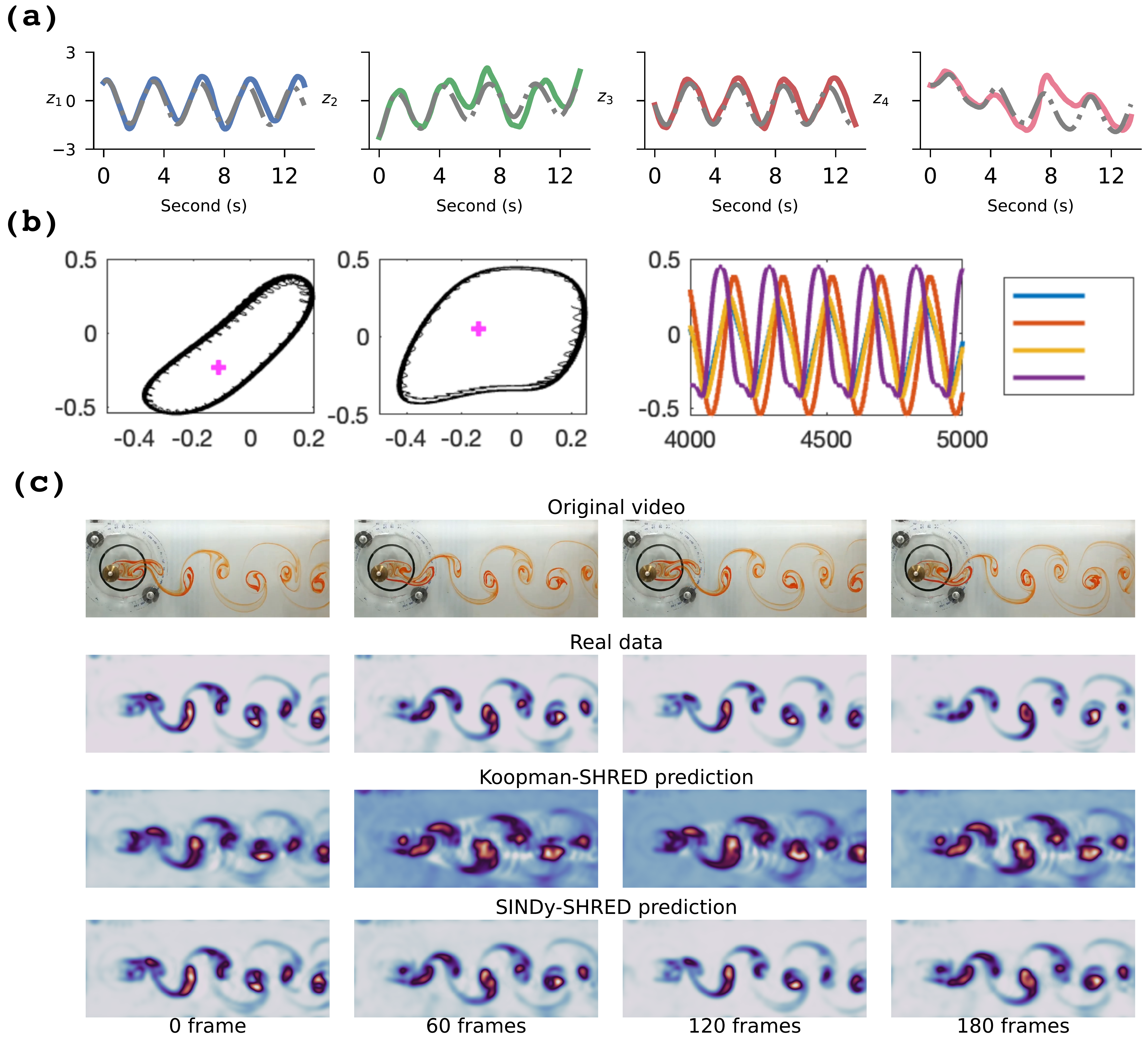}
    \put(17,50.5){$z_1$}
    \put(5,66){$z_2$}
    \put(42,50.5){$z_3$}
    \put(29,66){$z_4$}
    \put(54,66){$z_j(t)$}
     \put(70,50){time}
     \put(95,65.3){$z_1$}
     \put(95,62.9){$z_2$}
    \put(95,60.5){$z_3$}
    \put(95,58.1){$z_4$}
    \end{overpic}
    \caption{(a) Extrapolation of latent representation in SINDy-SHRED from the discovered dynamical system for flow over a cylinder data. Colored: true latent representation. Grey: SINDy extrapolation.  (b) Evolution of dynamical system (\ref{eqn:flow_equation}).  The left two panels are the phase plane $z_1$ vs $z_2$ and $z_3$ vs $z_4$ respectively.  The magenta plus symbols are the fixed points.  The nonlinear limit cycle behavior is shown in the right panel for $z_j(t)$. (c) Long-term pixel space video prediction via SINDy-SHRED. We demonstrate the forward prediction outcome up to 180 frames. }
    \label{fig:flow_latent_space}
\end{figure}
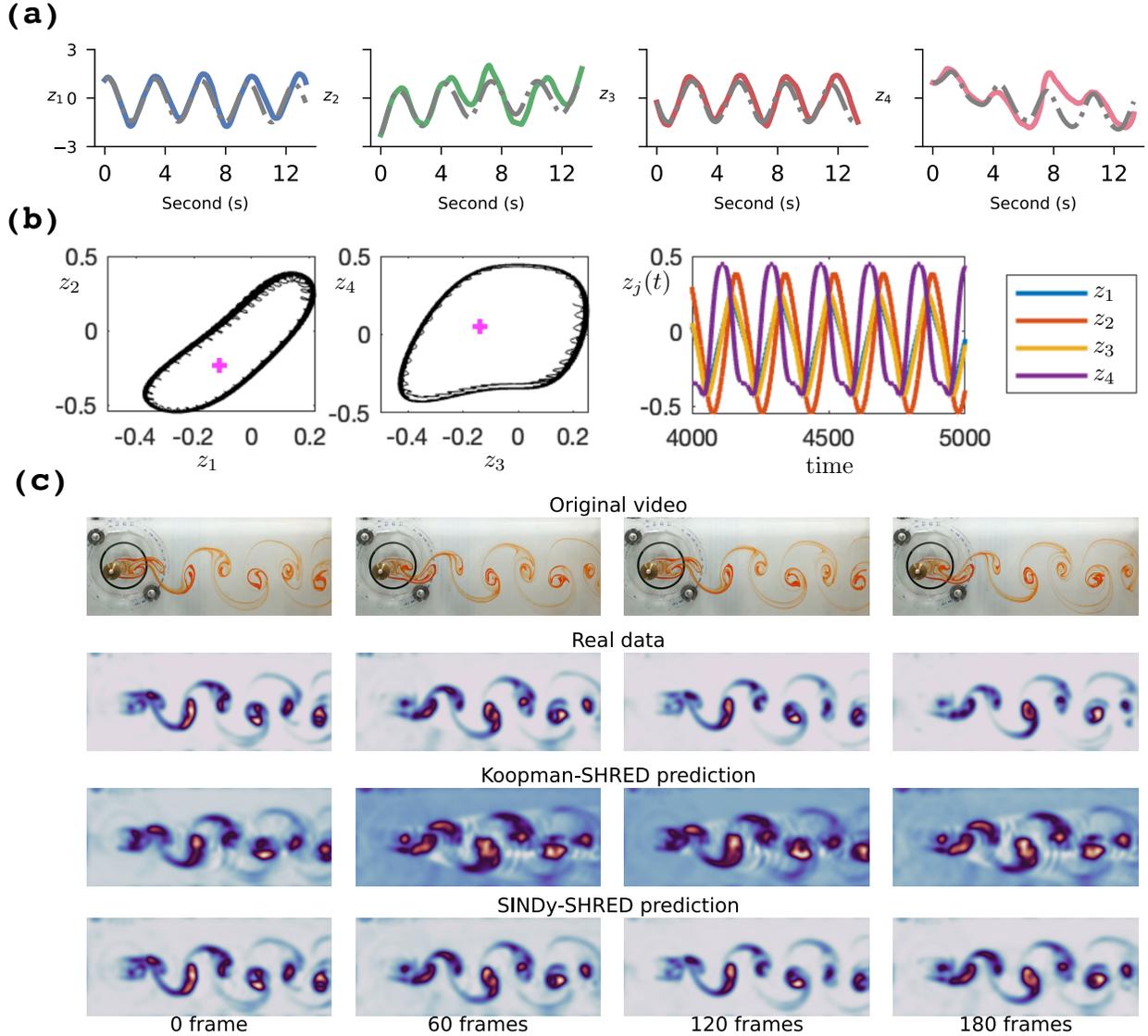

\noindent \textbf{Koopman-SHRED discovery} We define the hidden latent states as $(z_1, z_2, z_3, z_4, z_5, z_6)$. We discover the following dynamical system: 
\begin{align}
\label{eqn:flow_equation_koopman}
\begin{cases}
    \dot{z}_1 & = -0.920 z_4 + 0.620 z_5, \\
    \dot{z}_2 & = -0.163 z_1 + 0.817 z_4 + 0.778 z_6, \\
    \dot{z}_3 & = -0.462 z_4 - 1.026 z_6, \\
    \dot{z}_4 & = 1.791 z_1 + 0.307 z_6, \\
    \dot{z}_5 & = -0.969 z_1 + 0.548 z_6, \\
    \dot{z}_6 & = -0.800 z_2 - 0.915 z_5.
\end{cases}
\end{align}
The analytic solution will have the form
\begin{equation}
    \mathbf z(t) = c_1 \mathbf v_1 \cos (\omega_1 t) e^{-\lambda _1 t} + c_2 \mathbf v_2 \sin (\omega_1 t) e^{-\lambda_1 t} + c_3 \mathbf v_3 \cos (\omega_3 t) e^{\lambda _3 t} + c_4 \mathbf v_4 \sin (\omega_3 t) e^{\lambda_3 t} + c_5 \mathbf v_5 e^{-\lambda_5 t} 
\end{equation}
where $\omega_1 = 1.52$, $\omega_3 = 1.05$, $\lambda_1 = 0.01$, $\lambda_3 = 0.11$, and $\lambda _5 = -0.20$. The complete explicit solution is given in \ref{sec:analytic_flow}.

Compared to the systems discovered in all previous examples, the flow over a cylinder model is much more complex with significant nonlinear interactions. 
In Eqn.~\ref{eqn:flow_equation}, we find that $z_1$ and $z_3$ behave like a governing mode of the turbulence swing; $z_2$ and $z_4$ further depict more detailed nonlinear effects. 
We further present a linear model derived from Koopman-SHRED in Eqn.~\ref{eqn:flow_equation_koopman} with its latent space evolution demonstrated in Fig.~\ref{fig:flow_latent_space_koopman}. 
We show the result of extrapolating this learned representation.
We generate the trajectory from the initial condition at time point 0 and perform forward integration for extrapolation.  
As shown in Fig.~\ref{fig:flow_latent_space} (a), the learned ODE closely follows the dynamics of $z_1$ and $z_3$ up to 7 seconds (210 timesteps); $z_2$ and $z_4$ also have close extrapolation up to 4 seconds. 
The Koopman-SHRED model closely follows the trend (in Fig.~\ref{fig:flow_latent_space_koopman}) but deviates more significantly from the governing dynamics.

This learned representation nicely predicts the future frames in pixel space. In SINDy-SHRED, the shallow decoder prediction has an averaged MSE error of $0.030$ (equivalently $3\%$) over the entire available trajectory. 
Koopman-SHRED has relatively worse prediction error with an averaged MSE error of $0.053$. 
In Fig.~\ref{fig:flow_latent_space} (c), we observe that the autoregressively generated prediction frames closely follow the true data, and further in Fig.~\ref{fig:flow_prediction_long}, we find that the predictions are still stable after 1,000 frames, which is out of the size of the original dataset. 
The sensor-level prediction in Fig.~\ref{fig:flow_sensor_predictions} further demonstrates the accuracy of reconstruction in detail.

\subsection{Prediction and baseline study of single shot pendulum video}

In this subsection, we compare the performance of SINDy-SHRED to existing state-of-the-art learning algorithms. 
In the following, we demonstrate the result of video prediction on the pendulum data using ResNet~\citep{he2016deep}, convolutional LSTM (convLSTM)~\citep{shi2015convolutional}, and PredRNN~\citep{wang2017predrnn}, and SimVP~\citep{gao2022simvp}. The pendulum in our experiment is not ideal and includes complex damping effects. 
We use a nail on the wall and place the rod (with a hole) on the nail. 
This creates complex friction, which slows the rod more when passing the lowest point due to the increased pressure caused by gravity. 
The full model we discovered from the video (as shown in Fig.~\ref{fig:pendulum_generation} (a)) includes four terms: 
\begin{align}
    \ddot{z}=0.17\dot{z}^2-0.06\dot{z}^3-10.87\sin(z)+0.48\sin(\dot{z}),
\end{align}
which includes complex damping effects via $\sin(\dot{z}), \dot{z}^2, \dot{z}^3$, in contrast to the commonly assumed linear damping.

As shown in Table~\ref{tab:pendulum_baseline}, SINDy-SHRED outperforms all baseline methods for total error and long-term predictions. 
Generally, all baseline deep learning methods perform well for short-term forecasting, but the error quickly accumulates for longer-term predictions. 
This is also observable from the prediction in the pixel space as shown in Fig.~\ref{fig:pendulum_generation} (b). SINDy-SHRED is the only method that does not produce collapsed longer-term predictions. 
Interestingly, we observe that incorporating nonlinear terms in the SINDy library is crucial in this case, as Koopman-SHRED struggles to accurately reconstruct the finer details.
We attribute this to the fact that strong linear regularization may over-regularize and oversimplify the complex dynamics, leading to a reduction in predictive accuracy. 
Similar behaviors for Koopman-SHRED are observed when encoding the continuous spectrum using an additional shallow network.
In Fig.~\ref{fig:pen_sensor_predictions}, the sensor level prediction also demonstrates the robustness of the SINDy-SHRED prediction. PredRNN is the second best method as measured by the total error. 
However, PredRNN is expensive in computation which includes a complex forward pass with an increased number of parameters. 
It is also notable that the prediction of PredRNN collapses after 120 frames, after which only an averaged frame over the entire trajectory is predicted. 
ConvLSTM has a relatively better result in terms of generation, but the long-term prediction is still inferior compared to SINDy-SHRED. 
Additionally, is should be noted that 2D convolution is much more computationally expensive. 
For larger spatiotemporal domains (e.g. the SST example and 3D ozone data), the computational complexity of convolution will scale up very quickly, which makes the algorithm nearly impossible to deploy in practice. 
Similar computational issues will occur for diffusion models and generative models, which is likely to be impractical to compute, and unstable for longer-term predictions. 
In summary, we observe that SINDy-SHRED is not only a more accurate long-term model, but is also faster to execute and smaller in size. 

\begin{figure}[t]
    \centering
    \includegraphics[width=0.9\textwidth]{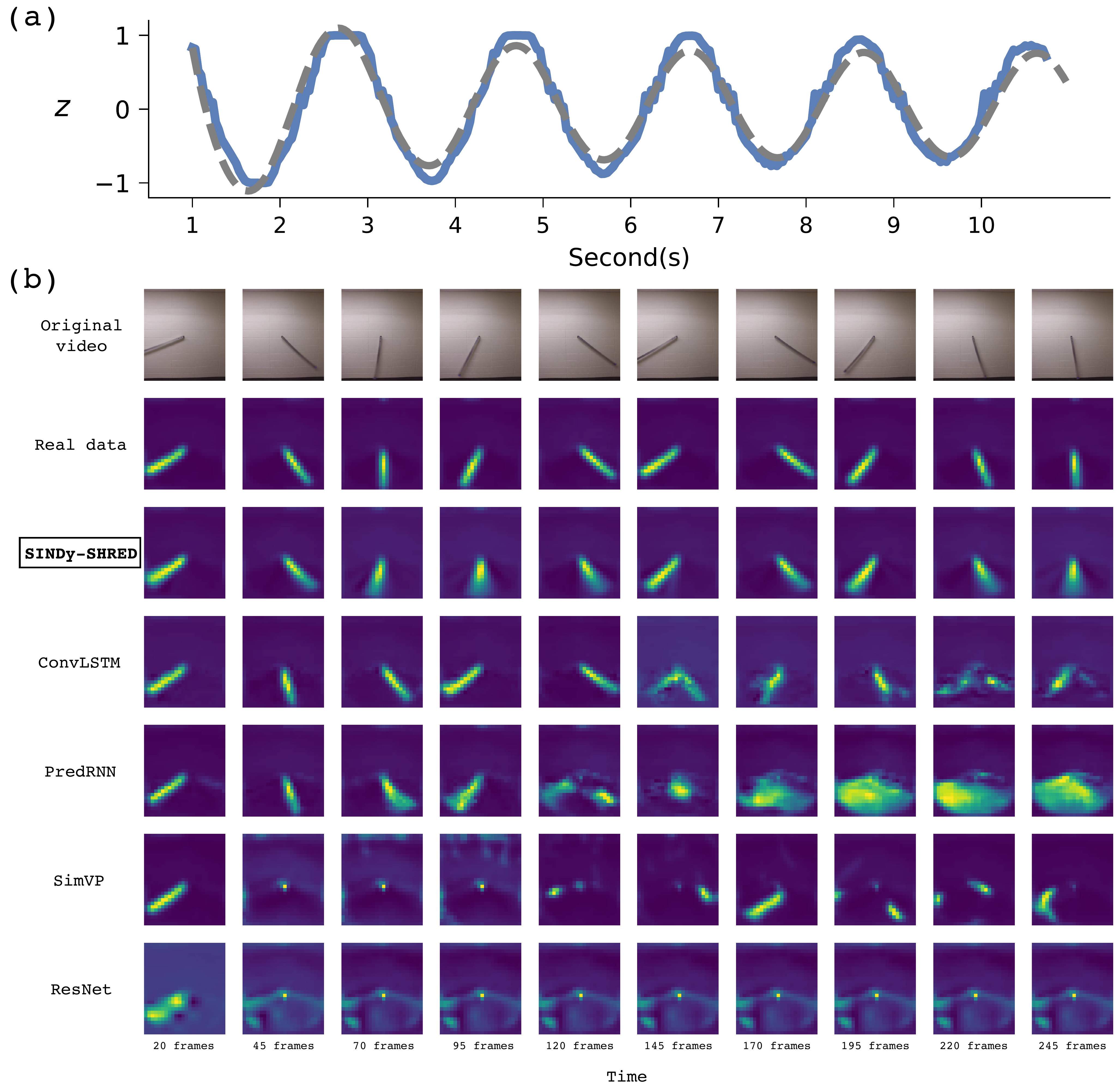}
    \caption{(a) Extrapolation of latent representation in SINDy-SHRED from the discovered dynamical system for the pendulum moving data. Blue: true latent representation. Grey: SINDy extrapolation. (b) The pendulum video generation outcome from ResNet, SimVP, ConvLSTM, PredRNN, and SINDy-SHRED from frame 20 to frame 245. }
    \label{fig:pendulum_generation}
\end{figure}

\begin{table}[t]
\centering
\scriptsize
\resizebox{\columnwidth}{!}{\begin{tabular}{lccccccc}
\toprule
Models & Params & Training time & $T=[0,100]$ &$T=[100, 200]$ & $T=[200,275]$ & Total \\
\midrule
ResNet~\citep{he2016deep} & 2.7M & 24 mins & $2.08\times 10^{-2}$ & $1.88\times 10^{-2}$ & $2.05\times 10^{-2}$ & $2.00\times 10^{-2}$  \\
SimVP~\citep{gao2022simvp}  & 460K & 30 mins &  $2.29\times 10^{-2}$ & $2.47\times 10^{-2}$ & $2.83\times 10^{-2}$ & $2.53\times 10^{-2}$  \\
PredRNN~\citep{wang2017predrnn}  & 444K & 178 mins & $1.02\times 10^{-2}$ & $1.79 \times 10^{-2}$ & $1.69 \times 10^{-2}$ & $1.48\times 10^{-2}$ \\
ConvLSTM~\citep{shi2015convolutional}  & 260K & 100 mins & $\mathbf{9.24}\times \mathbf{10^{-3}}$ & $1.86\times 10^{-2}$ & $1.99\times 10^{-2}$ & $1.55\times 10^{-2}$ \\
\midrule
\textbf{SINDy-SHRED}$^*$ & \textbf{44K} & \textbf{17 mins} & $1.70\times 10^{-2}$ & $\mathbf{9.36}\times \mathbf{10^{-3}}$ & $\mathbf{5.31}\times \mathbf{10^{-3}}$ & $\mathbf{1.05}\times \mathbf{10^{-2}}$\\
\bottomrule
\end{tabular}
}
\vspace{-.5em}
\caption{Comparison table of SINDy-SHRED to baseline methods for parameter size, training time, and mean-squared error over different prediction horizons.}  
\label{tab:pendulum_baseline}
\end{table}

\subsection{NOAA sea-surface temperature data}

\label{sec:sst_expr}
The third example we consider is that of global sea-surface temperature. 
The SST data contains 1,400 weekly snapshots of the weekly mean sea surface temperature from 1992 to 2019 reported by NOAA~\citep{reynolds2002improved}. 
The data is represented by a 180 $\times$ 360 grid, of which 44,219 grid points correspond to sea-surface locations. 
We standardize the data with its own min and max, which transforms the sensor measurements to within the numerical range of $(0, 1)$. 
We randomly select 250 sensors from the possible 44,219 locations and set the lag parameter to 52 weeks.  
Thus, for each input-output pair, the input consists of the 52-week trajectories of the selected sensors, while the output is a single temperature field across all 44,219 spatial locations. 
We include the details of the experimental settings of SINDy-SHRED in the Appendix~\ref{app:sst_expr_detail}. 
From the discovered coordinate system, we define the representation of physics - the latent hidden state space - to be $(z_1, z_2, z_3)$. The dynamics progresses forward via the following set of equations: 
\begin{align}
\label{eqn:sst_equation}
\begin{cases}
    \dot{z}_1 & = 4.68 z_2 -2.37 z_3, \\
\dot{z}_2 & = -3.10 z_1 + 3.25 z_3, \\
\dot{z}_3 & = 2.72 z_1 -5.55 z_2, 
\end{cases}
\end{align}
where the analytic solution to this system of ODEs will have the form 
\begin{equation}
    \mathbf z(t) = c_1 \mathbf v_1 \cos (\omega_1 t) e^{-\lambda _1 t} + c_2 \mathbf v_2 \sin (\omega_1 t) e^{-\lambda_1 t} + c_3 \mathbf v_3 e^{\lambda_3 t}.
\end{equation}
The value of $\omega_1 = 1.99\pi$, approximately corresponding to the expected period of 1 year. $\lambda_1 = 0.00763$ indicating a slow decay of the oscillatory mode with half-life $90.84$ years.  $\lambda_3 = 0.01527$ indicating global increases in temperature with doubling time $45.39$ years. The explicit solution is given in Appendix \ref{sec:analytic_sst}.

\begin{figure}[t]
    \centering
    \includegraphics[width=\textwidth]{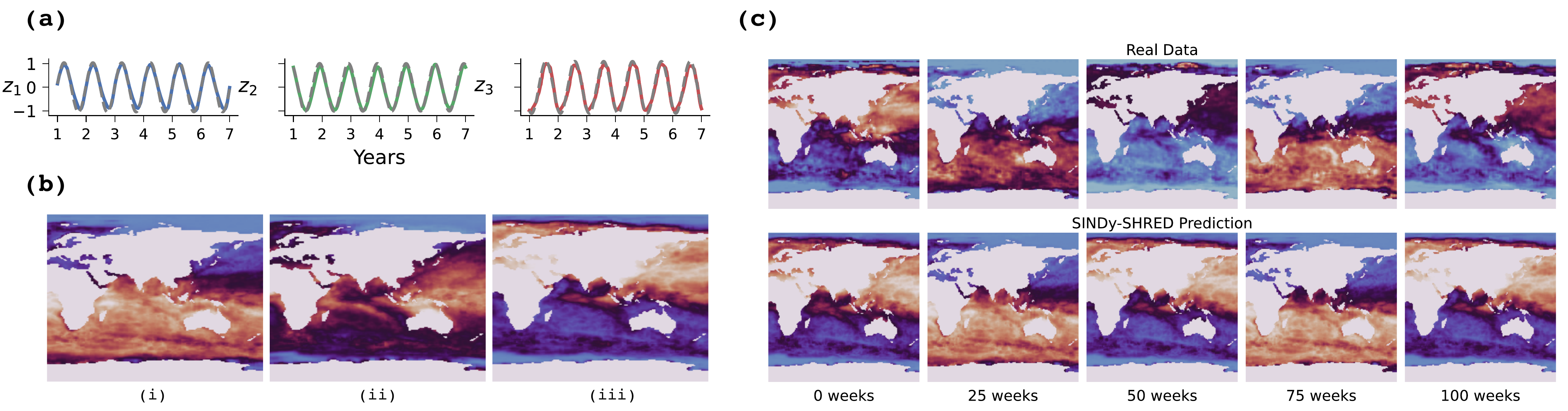}
    \caption{(a) Extrapolation of latent representation in SINDy-SHRED from the discovered dynamical system for SST. Colored: true latent representation. Grey: SINDy extrapolation. (b) Decoder reconstruction of three independent directions $z_1, z_2, z_3$ in the latent space. (c) Long-term global sea-surface temperature prediction via SINDy-SHRED from week 0 to week 100. We crop the global map for better visualization. }
    \label{fig:sst_combined}
\end{figure}

The discovery of a linear system describing the evolution of the latent state is in line with prior work on SST data \cite{de2019deep}
in which it as assumed that the underlying physics is an advection-diffusion PDE.  
In Fig.~\ref{fig:sst_combined} (a) we further present the accuracy of this discovered system by forward simulating the system from an initial condition for a total of 27 years (c.f. Fig.~\ref{fig:sst_latent_space_long_term}). It is observable how the discovered law is close to the true evolution of latent hidden states and, critically, there appears to be minimal phase slipping. By extrapolating the latent state space via forward integration, we can apply the shallow decoder to return forecasts of the future spatial domain. Doing so, we find an averaged MSE error of $0.57 \pm 0.10 ^\circ C$ for 318 time steps in the test dataset. The eigenmodes of the SST data is visualized by feeding unit vectors along the independent latent directions $z_1, z_2, z_3$ into the decoder in Fig.~\ref{fig:sst_combined}~(b). 
In Fig.~\ref{fig:sst_combined} (c), we show SINDy-SHRED produces stable long-term predictions for SST data. 
We further include Fig.~\ref{fig:sst_sensors} to demonstrate the extrapolation of each sensor. The sensor level prediction is based on the global prediction of the future frame, and we visualize the signal trajectory of randomly sampled testing sensor locations.
We find SINDy-SHRED is robust for out-of-distribution sensors, with reasonable extrapolations even in the presence of anomalous events.

\subsection{Predicting isotropic turbulence flow}

\begin{figure}[t]
    \centering
    \includegraphics[width=\textwidth]{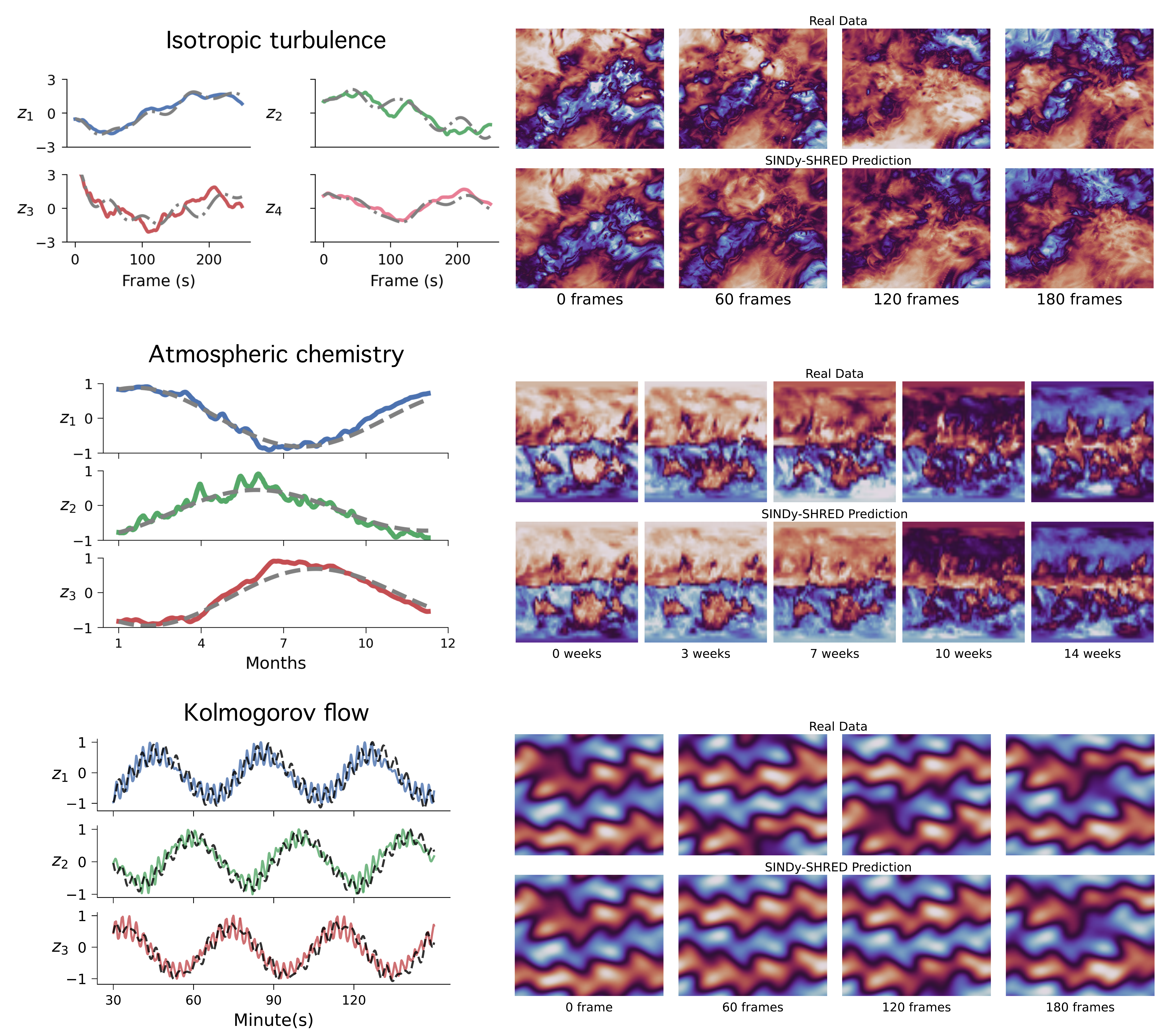}
    \caption{Visualizations of latent spaces and predictions for SINDy-SHRED applied to PDE simulations of isotropic turbulence, atmospheric chemistry, and 2D Kolmogorov flow datasets. \textbf{Left:} Extrapolation of latent representations using SINDy-SHRED from the discovered dynamical system for the ozone data, showing the true latent representation (colored lines) and SINDy extrapolation (grey lines). \textbf{Right:} Long-term spatiotemporal predictions from SINDy-SHRED compared to ground truth. The details of the 2D Kolmogorov experiment can be found in the Appendix.}
    \label{fig:flow_examples}
\end{figure}

The isotropic turbulance flow data~\citep{li2008public,yeung2012dissipation} is a data simulation from the pressure field of a forced isotropic turbulent flow. 
\begin{align}
    \nabla^2(p/\rho)=\lrp{\Omega-\epsilon/\nu}/2,
\end{align}
where $p$ is the pressure, $\rho$ is the density of the fluid, $\Omega$ is the vorticity, $\epsilon$ is the dissipation rate, and $\nu$ is the kinematic viscosity.
The flow was generated with $1024^3$ nodes with the Taylor-scale Reynolds number fluctuates around $433$. 
We process the data to shrink the spatial resolution to $350\times 350$, while keeping the time step to be $0.002$. 
We standardize the data within the range of $(0,1)$ with 50 sensors out of $122,500$ grid points (0.4\%). We include details of SINDy-SHRED setup in the Appendix~\ref{app:iso_expr_detail}. 

We find an approximation of the converged latent representation which has 8 oscillatory modes. 
It has the following analytic solution: 
\begin{align}
    \label{eqn:iso_equation}
    \mathbf{z}(t) = \sum_{i=1}^{7} \left( c_{2i-1} \mathbf{v}_{2i-1} \cos(\omega_i t) e^{\lambda_i t} + c_{2i} \mathbf{v}_{2i} \sin(\omega_i t) e^{\lambda_i t} \right) + c_{15}\vv_{15} e^{\lambda_8t} + c_{16}\vv_{16} e^{\lambda_9t}, 
\end{align}
where $\omega_1 = 9.39$, $\lambda_1 = 0.47$, $\omega_2 = 11.90$, $\lambda_2 = 0.05$, $\omega_3 = 13.42$, $\lambda_3 = 0.03$, $\omega_4 = 3.46$, $\lambda_4 = -0.04$, $\omega_5 = 5.44$, $\lambda_5 = -0.27$, $\omega_6 = 8.30$, $\lambda_6 = -0.75$, $\omega_7 = 18.72$, $\lambda_7 = 1.37$, $\lambda_8 = -0.26$, $\lambda_9 = -3.39$, with coefficients $c_{i}$' and eigenvectors $\mathbf{v}_{i}$. Please find the exact form in Eqn.~\ref{eqn:iso_equation_full}.

The isotropic turbulent flow exhibits chaotic behavior which makes the prediction challenging. However, a linear model still captures the governing dynamics nicely and provides a stable prediction for approximately 200 frames. 
In Eqn.~\ref{eqn:iso_equation}, we find three increasing oscillatory modes and five decreasing oscillatory modes. 
For pixel space prediction, SINDy-SHRED prediction has an averaged MSE error of $0.03$ over the entire dataset. 
In Fig.~\ref{fig:flow_examples}, we find that the predictions are accurate up to 180 frames, which closely follows the governing trend of turbulent flow. 
The sensor-level prediction in Fig.~\ref{fig:iso_sensor_predictions} further demonstrates the details of the signal prediction.
Three of the oscillatory modes demonstrate growing trends. 
Four of the oscillatory modes decay, as well as the exponential modes. 
We note here that due to the complex and chaotic nature of the data, accurate sensor-level prediction is particularly challenging. Despite this complexity, the SINDy-SHRED model still provides a stable extrapolation of the overall trends.

\subsection{3D Atmospheric ozone concentration}

The atmospheric ozone concentration dataset~\citep{bey2001global} contains a one-year simulation of the
evolution of an ensemble of interacting chemical species through a transport operator using GEOS-Chem. 
The simulation contains 1,456 temporal samples with a timestep of 6 hours over one year for 99,360 (46 by 72 by 30) spatial locations (latitude, longitude, elevation). The data presented in this work has compressed by performing an SVD and retaining only the first 50 POD modes.
As with the SST data, we standardize the data within the range of $(0,1)$
and randomly select and fix 3 sensors out of 99,360 spatial locations ($0.5\%$).
We include the details of the experimental settings of SINDy-SHRED in the Appendix~\ref{app:ozone_expr_detail}. 
The converged latent representation presents the following SINDy model: 
\vspace{-1mm}
\begin{align}
\label{eqn:ozone_equation}
\begin{cases}
\dot{z}_1 & = -0.002-0.013z_2+0.007z_3, \\
\dot{z}_2 & = -0.001 z_1 + 0.004 z_2 -0.008 z_3, \\
\dot{z}_3 & = 0.002 + 0.012 z_2 -0.005 z_3. 
\end{cases}
\end{align}
The analytic solution to this system of ODEs will have the form 
\begin{align}
    \mathbf z(t) &= c_1 \mathbf v_1 \cos (\omega_1 t) e^{-\lambda _1 t} + c_2 \mathbf v_2 \sin (\omega_1 t) e^{-\lambda_1 t} + c_3 \mathbf v_3 e^{-\lambda_3 t} \\ &+ \int _0 ^t \left( c_4 \mathbf v_1 \cos (\omega_1 t) e^{-\lambda _1 t} + c_5 \mathbf v_2 \sin (\omega_1 t) e^{-\lambda_1 t} + c_6 \mathbf v_3 e^{-\lambda_3 t} \right) d\tau 
\end{align}
where $\omega _1 = 0.0079,$ $\lambda_1 = 0.003$, and $\lambda_3 = 0.003.$ The complete solution is given in Appendix \ref{sec:analytic_ozone}.

Unlike traditional architectures for similar problems, which may include expensive 3D convolution, SINDy-SHRED provides an efficient way of training, taking about half an hour. Although the quantity of data is insufficient to perform long term-predictions, SINDy-SHRED still exhibits interesting behavior for a longer-term extrapolation which converges to the fixed point at $\mathbf{0}$ (as shown in Fig.~\ref{fig:ozone_latent_long_term}). From the extrapolation of the latent state space, the shallow decoder prediction has an averaged MSE error of $1.5e^{-2}$. In Fig.~\ref{fig:flow_examples}, we visualize the shallow decoder prediction up to 14 weeks. 
In Fig.~\ref{fig:flow_examples}, we reconstruct the sensor-level predictions which demonstrate the details of the signal prediction. The observations are much noiser than the SST data, but SINDy-SHRED provides a smoothed extrapolation for the governing trends.

\subsection{Visualizing the convex loss landscape}

We visualize the loss landscape for each experiment in Fig.~\ref{fig:loss_landscape}. 
Following the methodology of~\citep{li2018visualizing}, this visualization demonstrates the loss landscape along two random directions. 
By perturbing the model weights along these directions, we obtain a 3D visualization of the loss surface. 
We observe that, for both SHRED and SINDy-SHRED, the loss landscape is surprisingly smooth and exhibits perturbative convexity, featuring a clear global minima. 
The visualizations across all datasets consistently demonstrate these favorable convex loss landscapes.
In the Appendix, we further visualize different settings of SHRED under various random directions and scales in Fig.~\ref{fig:loss_landscapes}.
From these optimization landscapes, we observe that the SHRED architecture encourages flat minimizers and reduces the potential for chaotic training behavior. 
This property also provides SHRED with properties including hyperparameter stability, reliable statistical inference, and guaranteed data recovery. 

\begin{figure}
    \centering
    \includegraphics[width=\linewidth]{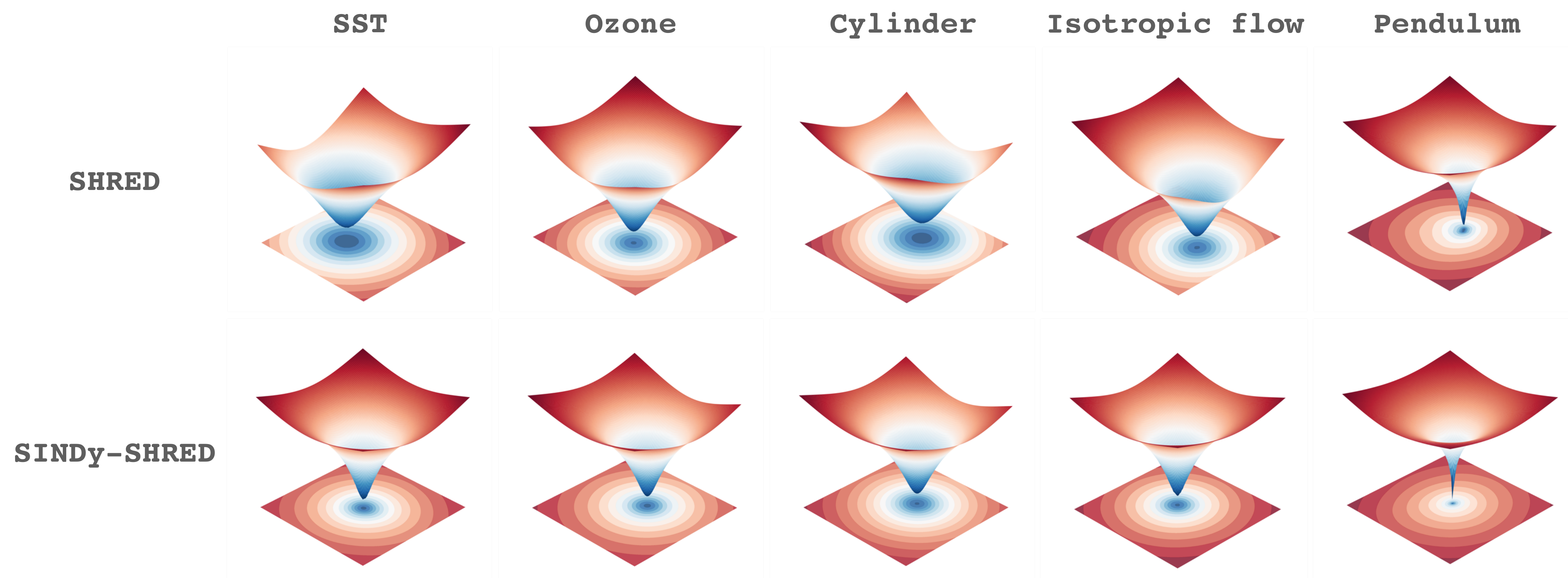}
    \caption{Loss landscape plot of SHRED (top row) and SINDy-SHRED (bottom row) for all datasets presented: sea-surface temperature, Ozone, flow over a cylinder, isotropic flow, and pendulum.}
    \label{fig:loss_landscape}
\end{figure}

\vspace{-3mm}
\section{Discussion and conclusion}
\vspace{-2mm}

In this paper, we present the SINDy-SHRED algorithm, which jointly solves the sensing and model identification problem with simple implementation, efficient computation, and robust performance.  Specifically, SINDy-SHRED (Koopman-SHRED) learns a latent coordinate system for which  parsimonious, low-dimensional governing equations, either linear or nonlinear, can be constructed.  The SINDy-SHRED model has strong predictive capabilities and stable roll outs by identifying the underlying physics. Through numerous experiments on challenging real-world data, we show that our method can produce robust and accurate long-term predictions for a variety of complex problems, including global sea-surface temperature, turbulent flows, and videos of physics, which includes recordings of flow over a cylinder and a moving pendulum.  The simplicity and robustness of the architecture, along with its performance on both real video and measurement streams, make it highly compelling as a new paradigm for learning physics models directly from real data streams.

Learning new physics and obtaining new insights from well-established physics problems is challenging, especially given the extensive efforts and long-standing interest from the broader community.  Despite this, SINDy-SHRED has been demonstrated to derive new physics models for every single example considered in this paper.  For instance, directly from video, SINDy-SHRED identified two potential models to capture the flow around a cylinder:  (i) a four-dimensional nonlinear ODE system containing a limit cycle modeling the nonlinear oscillations induced by the von Karman vortex shedding, and (ii) a linear six-dimensional Koopman model which can be solved in closed form.  Both models are new to the field, a rarity given the decades of research devoted to understanding the bifurcation that leads to the vortex shedding dynamics.  Moreover, the physics discovery is done solely from pixel space, a remarkable achievement that stands in contrast to previous efforts requiring access to the full PDE simulation fields.  In the pendulum example, SINDy-SHRED learns a nonlinear damping term that captures the complex motion of the system, improving on the typical linear damping assumption.  Beyond video data, SINDy-SHRED also demonstrates exceptional performance in modeling turbulent flows.  Remarkably, the latent space dynamics reveal that eight modes suffice to accurately capture the dynamics of isotropic turbulence.  Even global sea-surface temperature is captured with a three-dimensional linear ODE system that one might solve in a first course on ODEs.  These results are astounding given the complexity of the data sets, highlighting the capability of SINDy-SHRED to extract new understanding and governing equations from well-known and important scientific problems.

In addition to learning interpretable new physics models, SINDy-SHRED is simple and robust to train.  As shown above and in the supplement, this largely stems from the observed globally convex loss landscape, which is rare among all existing deep learning models as noted in~\citep{li2018visualizing}.  
In general, a major goal in designing deep learning models is to {\em engineer} the loss landscape to align with the inherent characteristics of the data, thereby facilitating stable training, robust hyper-parameter tuning, and generalizable performance~\citep{wang2021eliminating,ding2024flat}.  SINDy-SHRED naturally exhibits an ideal loss landscape that is globally convex, which is what the deep learning community strives to engineer.  
Although exploring the underlying mechanisms is beyond the scope of the current study, future work will investigate how the SINDy-SHRED architecture induces this ideal loss landscape.  Moreover, SINDy-SHRED is grounded by rigorous convergence guarantees for the SINDy-class of models, offering theoretical advantages over competing neural network-based approaches. 
These theoretical results offer precise error bounds and the performance for long-term prediction. 
As a result, SINDy-SHRED is well-grounded in theory, providing a solid foundation for further development and widespread practical applications.

Most importantly, in practice, SINDy-SHRED achieves state-of-the-art performance in long-term autoregressive video prediction, outperforming ConvLSTM, PredRNN, ResNet, and SimVP while maintaining the lowest computational cost, smallest model size and minimal training time  simultaneously.   Notably, SINDy-SHRED outperforms the leading methods on all metrics of interest without any hyper-parameter tuning. Moreover, this performance is achieved while producing an interpretable model of the underlying dynamics, effectively capturing the actual governing physics, a contribution that competing methods simply do not offer.  This enables stable, long-time rollouts of video prediction.   Such performance advances the goal of GoPro Physics, which aims to discover governing physics and engineering principles directly from few-shot video data streams.

To conclude, the SHRED architecture and its variants have proven to be highly effective algorithms for sensing~\cite{williams2023sensing,ebers2024leveraging}, parametric model reduction~\cite{tomasetto2025reduced,kutz2024shallow,riva2024robust} and now, as demonstrated in this work, physics discovery.  
The SINDy-SHRED and Koopman-SHRED architectures are robust, stable, and computationally efficient. This allows for laptop level computing along with afternoon-level tuning-even on exceptionally challenging datasets.  
Its ease of use and reliability position it as a compelling algorithm to be the new paradigm for general-purpose physics model discovery for a wide variety of datasets. 
Our inclusion of a Google colab notebook is aimed to assist practioners to extend SINDy-SHRED to perform physics discovery for custom data.  We have also included a YouTube playlist showcasing the diversity of applications and simplicity of implementation for both SHRED and SINDy-SHRED architectures:  \url{https://youtu.be/kHFQjSstNdQ}.

\section*{Acknowledgements}

The authors were supported in part by the US National Science Foundation (NSF) AI Institute for Dynamical Systems (dynamicsai.org), grant 2112085.  JNK further acknowledges support from the Air Force Office of Scientific Research (FA9550-24-1-0141).

\newpage
\bibliography{main.bbl}
\bibliographystyle{plainnat}

\newpage 
\appendix
\newpage

\begin{center}
{\LARGE \bfseries Supplemental Material}\\[0.5cm]
{\large Mars Liyao Gao, Jan P. Williams, J. Nathan Kutz}\\[0.5cm]
\end{center}

\section{Challenges in rolling out neural networks for fitting a simple sine function}

In the following example, we consider a simple use case in which we fit a simple sine function using recurrent neural networks. Surprisingly, extrapolating a simple sine function can be a difficult task for neural networks. The dynamical system $x$ is generated via the following equation: 
\begin{align}
    \ddot{x}=-\sin(x).
\end{align}

We implement a GRU network in the following, which contains three stacked GRU layers with size 500, and a fully connected output layer. 
We employ the Adam optimizer with a learning rate of 0.001 and used the mean squared error (MSE) as the loss function. We train the GRU network with 150 epochs. The input sequences are normalized to [0, 1].

\begin{figure}[H]
    \centering
    \includegraphics[width=\textwidth]{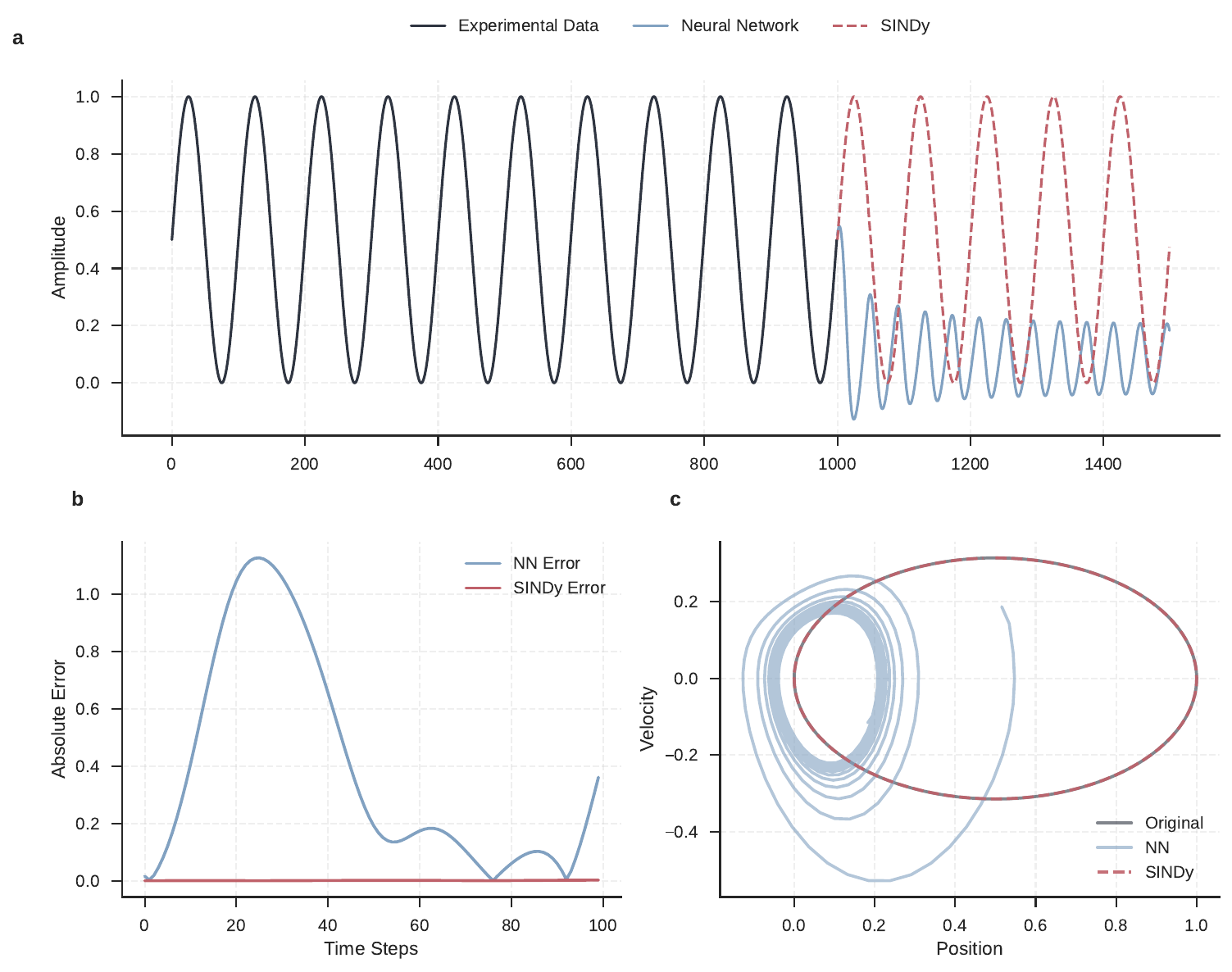}
    \caption{Fitting a simple sine function using SINDy and the GRU network.}
    \label{fig:GRU_sine}
\end{figure}

From Figure.~\ref{fig:GRU_sine}, we observe that SINDy approximation largely outperforms the GRU network. 
The extrapolation of GRU is poor due to the overly strong expressive power of this type of functional. 
We locate this to be a critical reason that most neural networks fail to capture physics nicely. 
Additionally, it is not clear for neural networks to estimate the dynamical system from the discrete setup $(t_0, t_1, ..., t_T)$ to a continuous one, while SINDy provides a natural transition into a continuous-time formulation.
The SINDy prediction could also remarkably improve generalizability, which could enhance the performance in long-term prediction.

\section{Theoretical justification on the generalization error between SINDy and the neural networks}
We further establish the statistical foundation for dynamical system learning.
When the underlying dynamical system can be closely described by a linear combination of the library of functions, obtaining a ``governing equation'' will have huge benefits for long-term extrapolation. 
yoDue to the nature of forward integration,error accumulates rapidly making an approximate system undesirable for extrapolation. In the following, we formalize this statement by analyzing the Rademacher complexity of SINDy-class and neural networks functional. 

\subsection{Theoretical setup}
The system we wish to study has the form that
\begin{align}
    \dot{x}=f(x),
\end{align}
which is an ODE describing the trajectory of a dynamical system in a learned latent space. 

Suppose we have measurements of $\vx=\{\vx_1, \vx_2, ...,\vx_t, \vx_{t+1},..., \vx_T\}$ with time gap $\Delta t$, and we wish to predict $\vx_{t+1}, \vx_{t+2}, \dots, \vx_T$. 

We first define the SINDy-class functional based on a library of functions $\Theta(\cdot)$. Mathematically, $\mathcal{F}_{\text{SINDy}}$ defines as: 
\begin{align}
    \mathcal{F}_{\text{SINDy}}:=\left\{\Theta\xi : \xi\in\real^p\right\},
\end{align}
where all $f \in \mathcal{F}_{\text{SINDy}}$ are functions of a  convex (linear) combination of functions in $\Theta(\cdot)$, and all functions in the library are within $\mathcal{L}^2(P)$. 
The space $\mathcal{L}^2(P)$ refers to the set of square-integrable functions with respect to the probability measure $P$.
An example of $\Theta(x)$ is $[x, x^2, x^3, x^4, \sin(x), \cos(x)]$, and this could represent a dynamical system with the following form 
\begin{align}
    \dot{x}=a_1x+a_2x^2+a_3x^3+a_4x^4+a_5\sin(x)+a_6\cos(x),
\end{align}
where $a_1, a_2, ..., a_6$ are constants.

We note that $\mathcal{F}_{\text{SINDy}}$ is a wide class of functions. Since $\Theta(\cdot)$ covers polynomials and the Fourier series, the functional class $\mathcal{F}_0$ could model the governing effect for all differentiable functions from Taylor's approximation. 

Then, we consider the neural network way of learning dynamical systems. 
Unlike SINDy learning, the neural network collects discrete samples $\vx=\{\vx_1, \vx_2, \dots, \vx_T\}$, and to predict $\vx_{t+1}$, we use the $r+1$ samples ahead. 
We consider the $k$-layer ReLU functional class $\mathcal{F}_{\text{ReLU}}$ to be 
\begin{align}
    \mathcal{F}_{\text{ReLU}}:=\left\{\vx\mapsto \sigma_k\lrp{\mW_k\sigma_{k-1}\lrp{\mW_{k-1}\cdots\sigma_1\lrp{\mW_1 \vx_{t-r:t}}}}\right\},
\end{align}
where $\theta=\lrp{\mW_1, \dots, \mW_k}$ with $W_1\in\real^{p\times d}$, $\mW_i\in\real^{d\times d}$, $\mW_k\in\real^{d\times p}$, and $\sigma_i(x)=\max(0, \vx)$ represents the ReLU function.
Here, $\vx_{t-r:t}$ represents of the prior $r+1$ samples. Instead of modeling thetemporal derivatives, this training strategy believes that the following dynamics are strongly tied with the history samples, and autoregressively predicts the future dynamics. 
This is commonly used in neural networks training for dynamical systems predictions.

For both cases, we define the error of dynamical system simulation as
\begin{align}
    \mathcal{E}(\vx,t)=\vx(t)-\hat{\vx}(t).
\end{align}

\paragraph{Stability}It is important to note that, in the following analysis, we set aside stability issues due to the lack of direct theoretical tools.
In practice, we observe that the dynamical systems approximated by neural networks are often exhibit instability.  
Longer rollouts of a learned dynamical system typically lead to poor performance. 
This is also reflected in the theory that it is hard to obtain guarantees without both a bounded input domain and bounded target values. 
There is no clear assurance on how neural networks might perform when getting out of the training input domain~\citep{shen2023engression}.
To facilitate the current analysis, we assume a stable forward simulation scheme by always projecting the $\hat{\vx}$ back into the input domain. 
Fully addressing the stability issue for neural networks will require more extensive study in future work.

\subsection{The error of SINDy-class}
In SINDy setting, the optimal functional $\theta^*(\cdot)\in\mathcal{F}_{\text{SINDy}}$ can be parameterized from the following: 
\begin{align}
    \theta^*(x)=\Theta(x)\xi^*,
\end{align}
where $\xi^*$ is a vector in $\mathbb{R}^p$. 
The functional learning problem shrinks down to a linear parameter estimation problem, and the empirical risk minimizer is the least-square solution that
\begin{align}
    \hat\xi = \lrp{\Theta(\mX)^T\Theta(\mX)}^{-1}\Theta(\mX)^\top \hat{\dot{\vx}}.
\end{align}

Here, we estimate the real $\dot{\vx}$ from the data using numerical differentiation methods via $\hat{\dot{x}}$. We assume a linear model for $\hat{\dot{x}}$ that  
\begin{align}
    \hat{\dot{\vx}} = \Theta(\mX)\xi^*+\epsilon,
\end{align}
where $\epsilon\sim\mathcal{N}(0, s^2)$. 

\begin{assumption}[Lowest eigenvalue]
    The lowest eigenvalue of $\Theta(\mX)^\top\Theta(\mX)$ is bounded from below by $cn$, where $c>0$, and $n$ is the sample size.
\end{assumption}

\begingroup
\def\thetheorem{\ref{thm:xi_error}}
\begin{theorem}
    Suppose we have a SINDy-class functional with a library of functions $\Theta(x)$, and we estimate the coefficient vector $\xi$ by $\hat{\xi}$ using least squares. Let $\xi^*$ be the true coefficient vector. Then, under suitable regularity conditions on $\Theta(x)$ and the noise in the data, we have
    \begin{align}
        E||\hat{\xi} - \xi^*||_2 = \mathcal{O}\lrp{s \sqrt{\frac{p}{n}}}
    \end{align}
    where $C$ is a constant depending on the properties of $\Theta(x)$, $s$ is the standard deviation of the noise, $p$ is the number of functions in the library $\Theta(x)$, and $n$ is the number of data points.

    The error in predicting the dynamical system after time $T$ is bounded by
    \begin{align}
        E||\hat{x}(T) - x(T)||_2 \leq  \mathcal{O}\left(e^{LT} E||\hat{\xi} - \xi^*||\right) = \mathcal{O}\left(e^{LT} s \sqrt{\frac{p}{n}}\right), 
    \end{align}
    where $L$ is the Lipschitz constant of the system.

    Furthermore, with probability $1-\delta$, the error in predicting the dynamical system after time $T$ is bounded by
    \begin{align}
        ||\hat{x}(T) - x(T)||_2 \leq  \mathcal{O}\left(e^{LT} s \sqrt{\frac{p}{n} \log\left(\frac{1}{\delta}\right)}\right).
    \end{align}
\end{theorem}
\addtocounter{theorem}{-1}
\endgroup

\begin{proof}
\begin{align}
    \E{\lrn{\hat\xi-\xi^*}_2^2} &= \E{\lrn{\lrp{\Theta(\mX)^T\Theta(\mX)}^{-1}\Theta(\mX)^\top \hat{\dot{\vx}} - \xi^*}_2^2} \\
    &= \E{\lrn{\lrp{\Theta(\mX)^T\Theta(\mX)}^{-1}\Theta(\mX)^\top (\Theta(\mX)\xi^* + \epsilon) - \xi^*}_2^2} \\
    &= \E{\lrn{\lrp{\Theta(\mX)^T\Theta(\mX)}^{-1}\Theta(\mX)^\top \epsilon}_2^2} \\
    &= \E{\text{Tr}\left[\lrp{\Theta(\mX)^T\Theta(\mX)}^{-1}\Theta(\mX)^\top \epsilon \epsilon^\top \Theta(\mX)\lrp{\Theta(\mX)^T\Theta(\mX)}^{-1}\right]} \\
    &= \text{Tr}\left[\lrp{\Theta(\mX)^T\Theta(\mX)}^{-1}\Theta(\mX)^\top \E{\epsilon \epsilon^\top} \Theta(\mX)\lrp{\Theta(\mX)^T\Theta(\mX)}^{-1}\right] \\
    &= s^2 \text{Tr}\left[\lrp{\Theta(\mX)^T\Theta(\mX)}^{-1}\Theta(\mX)^\top  \Theta(\mX)\lrp{\Theta(\mX)^T\Theta(\mX)}^{-1}\right] \\
    &= s^2 \text{Tr}\left[\lrp{\Theta(\mX)^T\Theta(\mX)}^{-1}\right]
\end{align}

Under suitable regularity conditions on $\Theta(x)$, the eigenvalues of $\Theta(\mX)^T\Theta(\mX)$ are well-behaved. If we assume that the smallest eigenvalue of $\Theta(\mX)^T\Theta(\mX)$ is bounded below by $c n$ for some constant $c > 0$, then the largest eigenvalue of $\lrp{\Theta(\mX)^T\Theta(\mX)}^{-1}$ is bounded above by $\frac{1}{cn}$. Therefore,

$$ \text{Tr}\left[\lrp{\Theta(\mX)^T\Theta(\mX)}^{-1}\right] \leq \frac{p}{cn}$$

where $p$ is the number of functions in $\Theta(x)$. Therefore,

$$\E{\lrn{\hat{\xi} - \xi^*}_2^2} \leq \frac{s^2 p}{cn}$$

Taking the square root, we have

$$\E{\lrn{\hat{\xi} - \xi^*}_2} \leq \frac{s \sqrt{p}}{\sqrt{cn}}.$$

The error is $\mathcal{E}(x,t) = x(t) - \hat{x}(t)$. Taking the derivative with respect to $t$, we get

\begin{align}
    \dot{\mathcal{E}}(x,t) &= \dot{x} - \dot{\hat{x}} \\
    &= \Theta(x) \xi^* - \Theta(x) \hat{\xi} \\
    &= \Theta(x) (\xi^* - \hat{\xi})
\end{align}

Since $\Theta(x)$ is $L$-Lipschitz, we have

$$||\dot{\mathcal{E}}(x,t)|| \leq L ||\hat{\xi} - \xi^*||$$

Using Grönwall's inequality, we get

$$||\mathcal{E}(x,T)|| \leq e^{LT} ||\mathcal{E}(x,0)|| + \int_0^T e^{L(T-t)} L ||\hat{\xi} - \xi^*|| dt$$

If we assume that the initial condition is perfect, i.e., $\mathcal{E}(x,0) = 0$, then

$$||\mathcal{E}(x,T)|| \leq \frac{e^{LT} - 1}{L} L ||\hat{\xi} - \xi^*|| \leq e^{LT} ||\hat{\xi} - \xi^*||$$

Therefore,

$$E||\hat{x}(T) - x(T)|| \leq  \mathcal{O}\left(e^{LT} E||\hat{\xi} - \xi^*||\right) = \mathcal{O}\left(e^{LT} s \sqrt{\frac{p}{n}}\right)$$

To get a high-probability bound, 
we wish to find a bound on $||\lrp{\Theta(\mX)^\top\Theta(\mX)}^{-1}\Theta(\mX)^\top\epsilon||_2$. 
Notice that this is a Gaussian random vector. 
Therefore, for any $\delta > 0$, with probability $1 - \delta$, we have

\begin{align}
\lrn{\lrp{\Theta(\mX)^\top\Theta(\mX)}^{-1}\Theta(\mX)^\top\epsilon}_2 \leq \sqrt{2 s^2 p \lambda_{\max}( \lrp{\Theta(\mX)^\top\Theta(\mX)}^{-1}) \log(1/\delta)}
\end{align}

By the lowest eigenvalue assumption, the smallest eigenvalue of $\Theta(\mX)^T\Theta(\mX)$ is bounded below by $cn$. Therefore, the largest eigenvalue of $\lrp{\Theta(\mX)^T\Theta(\mX)}^{-1}$ is bounded above by $\frac{1}{cn}$.

Substituting this into the bound, we get that with probability at least $1 - \delta$,
$$||\hat{\xi} - \xi^*||_2 \leq \sqrt{\frac{2 s^2 p}{cn} \log(1/\delta)} = s \sqrt{\frac{2 p\log(1/\delta)}{cn}}.$$

Therefore, we have that with probability at least $1-\delta$, 
\begin{align}
    \lrn{\hat{x}(T) - x(T)}_2 \leq  \mathcal{O}\left(e^{LT} s \sqrt{\frac{p\log\left(\frac{1}{\delta}\right)}{n}}\right).
\end{align}

\end{proof}

Although the error bound grows exponentially with $e^{LT}$, this growth can be practically small because $\Theta(\vx)$ is bounded, resulting in a small Lipschitz constant $L$ locally. This makes the error practically behave on the order of $\mathcal{O}\left(LTs\sqrt{\frac{p}{n}}\right)$.

\subsection{The error of neural networks}

In a feed-forward network setting, the optimal functional $\theta^*(\cdot)\in\mathcal{F}_{\text{ReLU}}$ can be parameterized from the following:
\begin{align}
    \theta^*(x)=\sigma_k\lrp{\mW_L^*\sigma_{k-1}\lrp{\mW_{k-1}^*\cdots\sigma_1\lrp{\mW_1^* \vx}}}.
\end{align}

This functional learning problem shrinks down to a parametric learning problem as well. The empirical rish minimizer is defined as
\begin{align}
    \hat\theta = \argmin_{\theta} \frac{1}{n}\sum_{i=1}^n \ell(f(\vx_i), \theta(\vx_i)),
\end{align}
where $\ell(\cdot)$ is the loss function, and $f(x_i)$ is the target value.

Apart from the SINDy-class, the empirical risk minimizer has to be solved via numerical optimization procedures. Additionally, as the loss landscape is non-convex, two-layer ReLU networks are prone to overfitting, and could perform poorly in extrapolation.
Fig.~\ref{fig:GRU_sine} shows how the extrapolation of neural networks can be suboptimal for learning dynamical systems. 
In the following, we present the error analysis for neural networks in dynamical system learning. 
\begingroup
\def\thetheorem{\ref{thm:nn_error}}
\begin{theorem}
    Suppose we have a neural network functional with $k$ layers of ReLU activation functions and parameters $\theta=\lrp{\mW_1, \dots, \mW_k}$, which computes functions
    \begin{align}
        f(\vx;\theta)=\sigma_k\lrp{\mW_k\sigma_{k-1}\lrp{\mW_{k-1}\cdots\sigma_1\lrp{\mW_1 \vx}}}.
    \end{align}
    Assume that the weights are bounded such that $\lrn{\mW_i}_F \leq B$ for all $i$.
        Then, with probability at least $1-\delta$, the error in predicting the dynamical system after $H = T/\Delta t$ steps is bounded by
    \begin{align}
        ||\hat{x}(T) - x(T)||_2 \leq  \mathcal{O}\left((\log n)^4 B^{k(H+1)} \sqrt{\frac{k}{n}} + \frac{\log(1/\delta)}{n}\right).
    \end{align}
\end{theorem}
\addtocounter{theorem}{-1}
\endgroup

\begin{proof}
    From Lemma~\ref{lemma:rademacher_generror}, with probability at least $1-\delta$, we have
    \begin{align}
        L(f) \leq (1+c)\hat{L}(f) + c(\log n)^4 \mathcal{R}_n(\mathcal{F}_{\text{ReLU}})+\frac{c\log(1/\delta)}{n},
    \end{align}
    where $L(f)$ is the true risk and $\hat{L}(f)$ is the empirical risk.

    We aim to bound the generalization error after $H = T/\Delta t$ steps. From Lemma~\ref{lemma:rademacher_nn}, the Rademacher complexity after $H$ steps is bounded by:
    \begin{align}
        \mathcal{R}_n(\mathcal{F}_{\text{ReLU}}\circ\dots\circ \mathcal{F}_{\text{ReLU}})\lesssim B^{kH}\frac{\sqrt{k}B^k}{\sqrt{n}} = B^{k(H+1)}\sqrt{\frac{k}{n}}.
    \end{align}

    Substituting this bound on the Rademacher complexity into the generalization error bound from Lemma~\ref{lemma:rademacher_generror}, we obtain:
    \begin{align}
        L(f) \leq (1+c)\hat{L}(f) + c(\log n)^4 B^{k(H+1)}\sqrt{\frac{k}{n}}+\frac{c\log(1/\delta)}{n}.
    \end{align}

    The true risk $L(f)$ represents the expected squared error over the true data distribution, and $\hat{L}(f)$ is the empirical risk on the training data.  Therefore, $L(f)$ can be interpreted as $E[||\hat{x}(T) - x(T)||^2]$.  However, we want to bound $||\hat{x}(T) - x(T)||$.

    Assuming that we have minimized the empirical risk $\hat{L}(f)$ sufficiently such that it is negligible, and taking the square root of the bound (and absorbing constants), we get that with probability at least $1-\delta$:
    \begin{align}
        ||\hat{x}(T) - x(T)||_2 \leq  \mathcal{O}\left((\log n)^2 B^{k(H+1)} \sqrt[4]{\frac{k}{n}} + \sqrt{\frac{\log(1/\delta)}{n}}\right).
    \end{align}
    However, a more direct application of Lemma 2.2 gives a better bound.  Specifically, the loss function is $(\hat{y}-y)^2$, which is bounded by 4 since $\hat{y}, y \in [-1, 1]$.  Thus, we can directly use the bound from Lemma 2.2.

    Therefore, with probability at least $1-\delta$:
    \begin{align}
        ||\hat{x}(T) - x(T)||_2 \leq  \mathcal{O}\left((\log n)^4 B^{k(H+1)} \sqrt{\frac{k}{n}} + \frac{\log(1/\delta)}{n}\right).
    \end{align}
\end{proof}

\subsubsection{Technical Lemmas for the Proof of Thm.~\ref{thm:nn_error}}
We introduce the following lemma from~\citep{bartlett2021deep} and rewrite it within our context. 

\begin{lemma}[Thm.~2.11 in~\citep{bartlett2021deep}]
    Let $\sigma:\real\to\real$ be the ReLU activation function. Consider a feed-forward network with $k$ layers of these nonlinearities and parameters $\theta=\lrp{\mW_1, \dots, \mW_L}$, which computes functions
    \begin{align}
        f(\vx;\theta)=\sigma_k\lrp{\mW_k\sigma_{k-1}\lrp{\mW_{k-1}\cdots\sigma_1\lrp{\mW_1 \vx}}}.
    \end{align}
    Define the class of functions on the unit Euclidean ball in $\real^p$, 
    \begin{align}
        \mathcal{F}_B=\left\{f(\cdot;\theta):\lrn{\mW_i}_F\leq B\right\},
    \end{align}
    where $\lrn{\mW_i}_F$ is the Frobenius norm of $\mW_i$.
        Then, we have the Rademacher complexity of $\mathcal{F}$ is bounded by
    \begin{align}
        \mathcal{R}_n(\mathcal{F}_B)\lesssim \frac{\sqrt{k}B^k}{\sqrt{n}}.
    \end{align} 
\end{lemma}

This result is from~\citep{golowich2018size} and~\citep{bartlett2021deep}. We then introduce the following connection between the Rademacher complexity and the generalization error.

\begin{lemma}[Thm.~2.2 in~\citep{bartlett2021deep}]
    \label{lemma:rademacher_generror}
    Let $\mathcal{F}_B$ be a class of neural networks with weights bounded by $B$, and let $f \in \mathcal{F}_B$. 
        For the mean-squared loss that $\ell(\hat{y}, y)=(\hat{y}-y)^2$, and for any distribution on $\mathcal{X}\times [-1,1]$ where the output is bounded. 
    With probability at least $1-\delta$, there exists a constant $c>0$ such that $\forall f\in\mathcal{F}_B$:
    \begin{align}
        L(f) \leq (1+c)\hat{L}(f) + c(\log n)^4 \mathcal{R}_n(\mathcal{F}_B)+\frac{c\log(1/\delta)}{n},
    \end{align}
    where $L(f)$ is the true risk, $\hat{L}(f)$ is the empirical risk, $n$ is the number of data points, and $B$ is a bound on the weights of the neural network.
\end{lemma}

An important observation here is that, even with a fixed set of parameters, for autoregressive neural networks, the rademacher complexity grows massively with longer time horizons. In other words, if we want to predict $T$ steps ahead, the network size will grow exponentially with $T$.
We use the following lemma called the contraction principle to analyze the expected network growth in this scenario. 

\begin{lemma}
    \label{lem:contraction}
    Let $\mathcal{F}_1$ be a class of functions from $\real^p$ to $\real^q$ with Lipschitz constant $\alpha$.
        Let $\mathcal{F}_2$ be a class of functions from $\real^q$ to $\real^r$.
        Then, the contraction principle describes that the Rademacher complexity of $\mathcal{F}_1\circ\mathcal{F}_2$ is 
    \begin{align}
        R_n(\mathcal{F}_1\circ \mathcal{F}_2) \leq \alpha\mathcal{R}_n(\mathcal{F}_2).
    \end{align}
\end{lemma}

The proof of this lemma can be found in~\citep{bartlett2002rademacher,ledoux2013probability}.
From this Lemma, we further control the Lipschitz constant for neural network in the following: 
\begin{lemma}[Proposition~2 in~\citep{virmaux2018lipschitz}]
    \label{lemma:lipschitz_nn}
        Suppose we have a neural network functional with $k$ layers of ReLU activation functions and parameters $\theta=\lrp{\mW_1, \dots, \mW_L}$, which computes functions
        \begin{align}
            f(\vx;\theta)=\sigma_k\lrp{\mW_k\sigma_{k-1}\lrp{\mW_{k-1}\cdots\sigma_1\lrp{\mW_1 \vx}}}.
        \end{align}
        The Lipschitz constant of the network is bounded by 
        \begin{align}
            \prod_{i=1}^k \lrn{\mW_i}_2 \leq B^k, 
        \end{align}
        where $\lrn{\mW_i}_2$ is the spectral norm of $\mW_i$.
\end{lemma}

\begin{proof}
    Since the ReLU activation function has a Lipschitz constant of 1, the Lipschitz constant of the $i$-th layer is bounded by $\lrn{\mW_i}_2$.
        Therefore, from Lemma~\ref{lem:contraction}, for a $k$-layer neural network with weights $\theta=\lrp{\mW_1, \dots, \mW_k}$, the Lipschitz constant of the entire network is bounded by the product of the spectral norms of the weight matrices:
    \begin{align}
        \text{Lip}(f(\vx;\theta)) \leq \prod_{i=1}^k \lrn{\mW_i}_2.
    \end{align}
        Given that $\lrn{\mW_i}_F\leq B$ for all $i$, and $\lrn{\mW_i}_2 \leq \lrn{\mW_i}_F$, we have
    \begin{align}
        \prod_{i=1}^k \lrn{\mW_i}_2 \leq \prod_{i=1}^k \lrn{\mW_i}_F \leq B^k.
    \end{align}
    This means that the Lipschitz constant of the neural network is bounded by $B^k$.
\end{proof}

From the Lemma above, we could obtain the Rademacher complexity of neural network prediction up to $H=\frac{T}{\Delta t}$ steps forward. 

\begin{lemma}
    The Rademacher complexity of the neural network functional class after $H$ forward unroll operations for $\mathcal{F}_{\text{ReLU}}$ is bounded by
    \begin{align}
        \mathcal{R}_n(\mathcal{F}_{\text{ReLU}}\circ\dots\circ \mathcal{F}_{\text{ReLU}})\lesssim B^{\frac{kT}{\Delta t}}\frac{\sqrt{k}B^k}{\sqrt{n}}.
    \end{align}
    \label{lemma:rademacher_nn}
\end{lemma}

\begin{proof}
    From Lemma~\ref{lemma:lipschitz_nn}, we know that the Lipschitz constant of the neural network is bounded by $B^k$. 
        Therefore, the Rademacher complexity of the neural network functional class is bounded by
    \begin{align}
        \mathcal{R}_n(\mathcal{F}_{\text{ReLU}})\lesssim \frac{\sqrt{k}B^k}{\sqrt{n}}.
    \end{align}
        From Lemma~\ref{lem:contraction}, we know that the Lipschitz constant of the network is bounded by $B^k$. 
        Therefore, the Rademacher complexity of the neural network functional class after $H$ forward unroll operations is bounded by
    \begin{align}
        \mathcal{R}_n(\mathcal{F}_{\text{ReLU}}\circ\dots\circ \mathcal{F}_{\text{ReLU}})\lesssim B^{\frac{kT}{\Delta t}}\frac{\sqrt{k}B^k}{\sqrt{n}}.
    \end{align}
\end{proof}

\section{Observation on the convex landscape}

We demonstrate our observations on the convexity of the loss landscape of SHRED. 
To visualize the loss landscape in a high-dimensional space, we utilize a popular method~\citep{li2018visualizing} which represents the landscape in the following way. 
Suppose the neural network is parameterized by $\theta$, which includes the weights and biases of all layers. 
We perturb $\theta_0$ with two random directions $\vr_x, \vr_y$ via 
\begin{align}
\theta'=\theta_0+t\cdot\alpha\vr_x+t\cdot\alpha\vr_y,
\end{align}
where $\vr_x, \vr_y$ are both normalized i.i.d. Gaussian samples. $\alpha$ is the scale of changes and $t\in [0,1]$ moves from $\theta_0$ towards the linear combination of the two random directions. 

In Fig.~\ref{fig:loss_landscapes}, we visualize the loss landscape of SHRED from different random directions using the methodology described above under various network depths and scales $\alpha$. Across all subplots, we consistently observe a clear global minimum with a convex landscape. 
Furthermore, Fig.~\ref{fig:convexity_analysis} presents an empirical convexity analysis, showing loss changes along several paths from different initializations. We evaluate convexity using randomly sampled points and find that all of them satisfy the convexity condition. 
We use a tolerance of $10^{-7}$ to counter floating-point imprecision. 

\begin{figure}[t]
    \centering
    \includegraphics[width=1.0\linewidth]{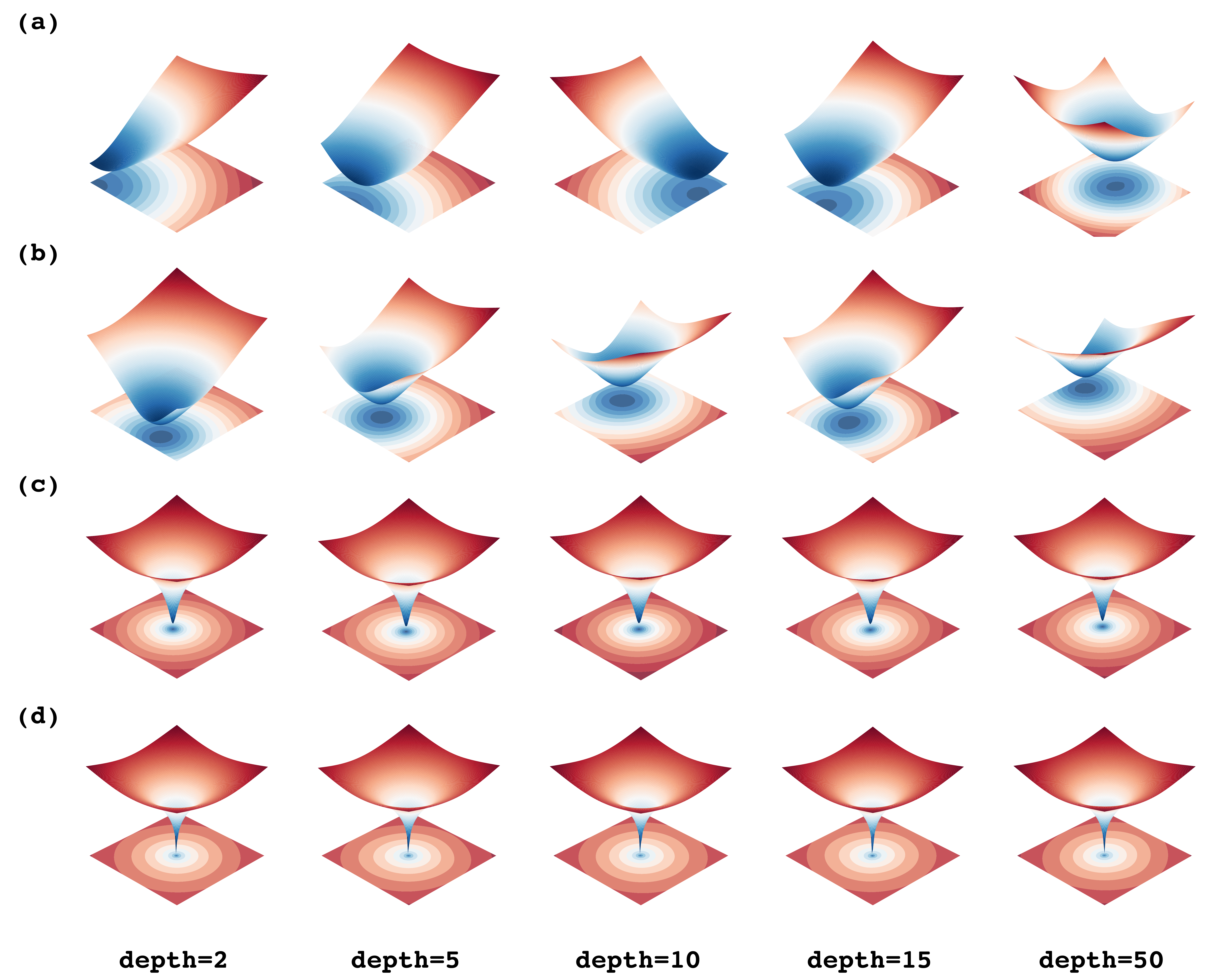}
    \caption{(a) Visualization of loss landscape from parameters ranging from $[-5,5]$. (b) Visualization of loss landscape from parameters ranging from $[-10,10]$. (c) Visualization of loss landscape from parameters ranging from $[-100,100]$. (d) Visualization of loss landscape from parameters ranging from $[-1000,1000]$. }
    \label{fig:loss_landscapes}
\end{figure}

\begin{figure}[t]
    \centering
    \includegraphics[width=1.0\linewidth]{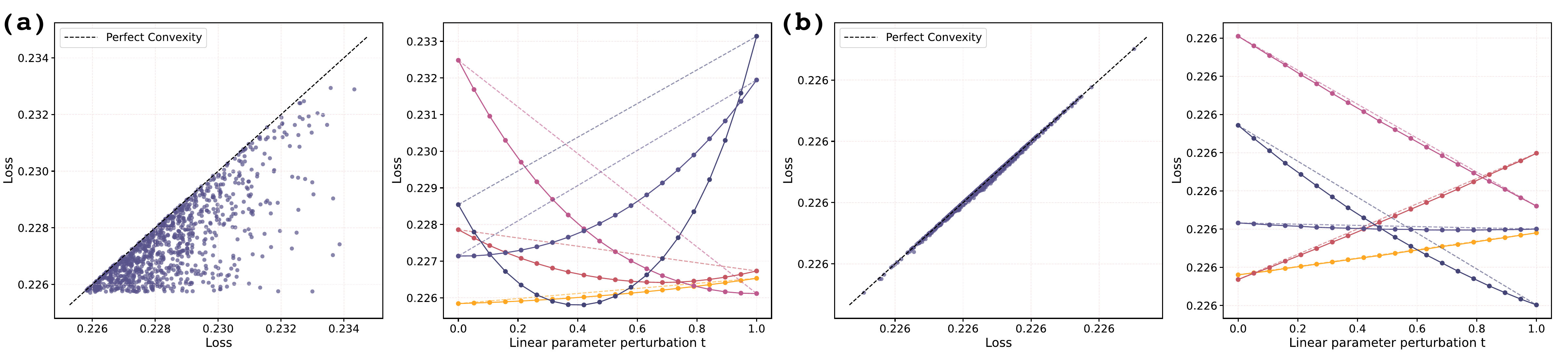}
    \caption{Visualization of loss landscapes for SHRED parameters along random directions, showing how loss values change when moving from the origin toward randomly generated points in parameter space. The figure provides a detailed view of loss behavior at different initial points. \textbf{(a)} shows the parameter range $[-100, 100]$ and \textbf{(b)} shows the parameter range $[-1,1]$.}
    \label{fig:convexity_analysis}
\end{figure}

\section{Experimental details}

\subsection{The design of SINDy unit}
In Fig.~\ref{fig:sindy_unit}, we show the design of SINDy unit. Notice how this design is similar to the skip-connection mechanism in ResNet. Here, it models the forward simulation process of ODE integration by having $\vx_{t+1}=\vx_i+\Theta(\vx)\Xi\Delta t$.

\subsection{Flow over a cylinder}
\label{app:flow_expr_detail}
In the flow over a cylinder experiment, we follow the same settings as in the prior experiments and select the latent dimension to be 4. 
The forward integration time step is set to $dt=\frac{1}{300}$ corresponding to the frame rate of 30 FPS. We set the batch size at $64$ and the learning rate to $5e^{-4}$. 
The thresholding procedure is executed every 300 epochs with thresholds ranging from $(1e^{-4}, 1e^{-3})$.  
From this extreme low-data limit, we manually perform data augmentation by reusing the latter part of the video once in the training loop to increase the number of available samples.

\subsection{Baseline experiment on pendulum}
\label{app:pen_expr_detail}
\paragraph{Autoregressive training.}
The raw pendulum data are collected from a 14-second GoPro recording. The raw data are present difficulties during training because of their high-dimensionality ($\1080 \times 960$), so we follow the same preprocessing procedure as in~\citep{mars2024bayesian} to obtain a set of training data with 390 samples, width 24 and height 27. For most of the models, we apply autoregressive training to help the model achieve better long-term prediction capabilities. 
From the initial input $\{\mX_1, \mX_2, ..., \mX_L\}$ with lag $L$, the model autoregressively predicts the next frame $\hat\mX_{L+1}$ and use it as a new input $\{\mX_2, \mX_3, ..., \hat\mX_{L+1}\}$. 
This step will be repeated $L$ times to obtain $\{\hat\mX_{L+1}, \hat\mX_{L+2}, ..., \hat\mX_{2L}\}$. We treat this as the prediction and optimize the loss from this quantity. 
In the following baseline models, we uniformly set $L=20$. 

\label{app:baseline_setting}
\paragraph{ResNet.} We use the residual neural network (ResNet)~\citep{he2016deep} as a standard baseline. 
We set the input sequence length to 20, and we predict the next frames autoregressively. 
For ResNet, the first convolutional layer has 64 channels with kernel size 3, stride 1 and padding 1.  
Then, we repeat the residual block three times with two convolutional layers. 
We use ReLU as the activation function.
After the residual blocks, the output is generated via a convolutional layer with kernel size 1, stride 1, and padding 0. 
We set the batch size to 8, and we use AdamW optimizer with learning rate $1e^{-3}$, weight decay $1e^{-2}$ for the training of 500 epochs. 
\paragraph{SimVP.} SimVP~\citep{gao2022simvp} is the recent state-of-the-art method for video prediction. 
This method utilizes ConvNormReLU blocks with a spatio-temporal features translator (i.e. CNN). 
The ConvNormReLU block has two convolutional layers with kernel size 3, stride 1, and padding 1. After 2D batch normalization and ReLU activation, the final forward pass includes a skip connection unit before output. 
The encoder first performs a 2D convolution with 2D batch normalization and ReLU activation. Then, three ConvNormReLU blocks will complete the input sequence encoding process. 
The translator in our implementation is a simple CNN which contains two convolutional layers. 
The decoder has a similar structure to the encoder by reversing its structure.  
We similarly set the batch size to 8 with AdamW optimizer for 500 epochs. 

\paragraph{ConvLSTM.} Convolutional Long Short-Term Memory~\citep{shi2015convolutional} is a classical baseline for the prediction of video sequence and scientific data (e.g. weather, radar echo, and air quality). 
The ConvLSTM utilizes features after convolution and performs LSTM modeling on hidden states. 
The ConvLSTM model has two ConvLSTM cells that have an input 2D convolutional layer with kernel size 3 and padding 1 before the LSTM forward pass. 
The decoder is a simple 2D convolution with kernel size 1, and zero padding. 
We similarly set the batch size to 8 with AdamW optimizer for 500 epochs. 

\paragraph{PredRNN.} PredRNN~\citep{wang2017predrnn} is a recent spatiotemporal modeling technique that builds on the idea of ConvLSTM. We follow the same network architecture setting as in ConvLSTM and similarly set the batch size to 8 with AdamW optimizer for 500 epochs. 
\paragraph{SINDy-SHRED.} We select and fix 100 pixels as sensor measurements from the entire $648$ dimensional space. We remove non-informative sensors, defined as remaining constant through the entire video. We set the lag to 60. 
For the setting of network architecture in SINDy-SHRED, we follow the same settings as in the prior experiments but with latent dimension of 1. 
The timestep of forward integration is set to $dt=\frac{1}{300}$ corresponding to frame rate of the video at 30 FPS. We set the batch size at $8$ and the learning rate to $5e^{-4}$. 
The thresholding procedure is executed every 300 epochs with thresholds ranging from $(0.4, 4.0)$.  
 We include 3 stacked GRU layers, and a two-layer ReLU decoder with 16 and 64 neurons. 
We use dropout to avoid overfitting with a dropout rate of 0.1. SINDy-SHRED discovers two candidate models. 

\subsection{Sea-surface temperature data}
\label{app:sst_expr_detail}
For the SST data in SINDy-SHRED, we set the latent dimension to 3 because we observe only minor impacts on the reconstruction accuracy when the latent dimension is $\geq 4$. We include 2 stacked GRU layers and consider the , and a two-layer ReLU decoder with 350 and 400 neurons. 
For the E-SINDy regularization, we set the polynomial order to be 3 and the ensemble number is 10. In the latent hidden-state forward simulation, we use Euler integration with $dt=\frac{1}{520}$, which will generate the prediction of next week via 10 forward integration steps. 
During training, we apply the AdamW optimizer with a learning rate of $1e^{-3}$ and a weight decay of $1e^{-2}$. The batch size is 128 with 1,000 training epochs. 
The thresholds for E-SINDy range uniformly from 0.1 to 1.0, and the thresholding procedure will be executed every 100 epochs. We use dropout to avoid overfitting with a dropout rate of 0.1. The training time is within 30 minutes from a single NVIDIA GeForce RTX 2080 Ti. 

\begin{figure}[t]
\centering
\includegraphics[width=0.5\linewidth]{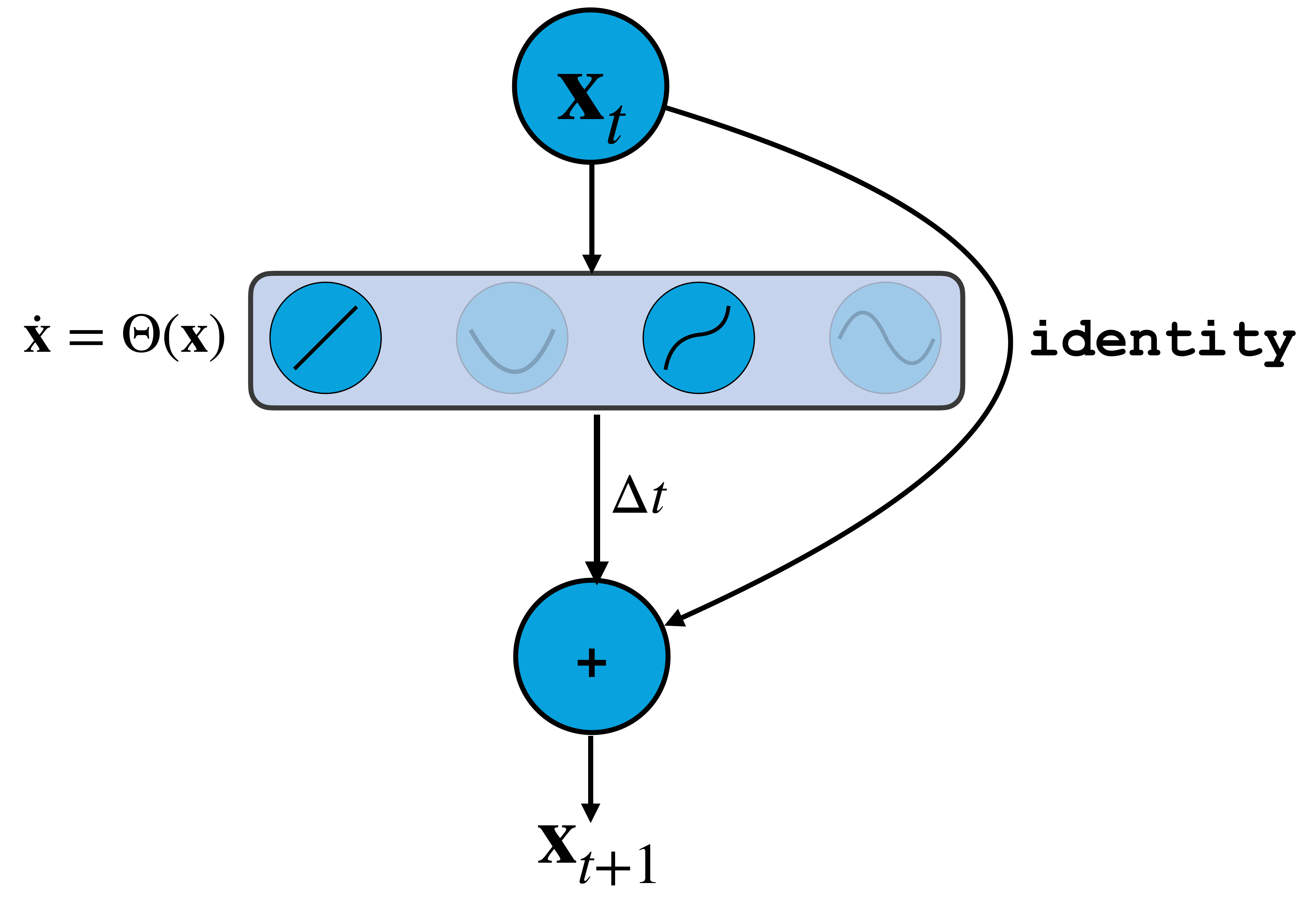}
\caption{Diagram of the RNN form
of SINDy.}\label{fig:sindy_unit}
\end{figure} 
\subsection{3D atmospheric ozone concentration}
\label{app:ozone_expr_detail}
For the ozone data, we set the lag parameter is set to 100. 
Thus, for each input-output pair, the input consists of the 62.5 day measurements of the selected sensors, while the output is the measurement across the entire 3D domain. 
In SINDy-SHRED, we follow the same network architecture as in the SST experiment. We set $dt=0.025$, and the thresholds for E-SINDy range uniformly from 0.015 to 0.15. 
The thresholding procedure will be executed every 300 epochs, and we apply AdamW optimizer with learning rate $1e^{-3}$.

\subsection{Isotropic turbulent flow}
\label{app:iso_expr_detail}
In the isotropic flow experiment, due to its complex nature, we select the latent dimension to be 16 with lag 100. 
The forward integration time step is set to $dt=0.0002$. We set the batch size to $128$ and the learning rate to $5e^{-4}$. 
The thresholding procedure is executed every 100 epochs with thresholds ranging from $(0.15, 1.5)$. The total training epoch is 300.

\section{Experiment on the 2D Kolmogorov flow}
\label{app:kolmogorov_flow}
The 2D Kolmogorov flow data is a chaotic turbulent flow generated from the pseudospectral Kolmogorov flow solver~\citep{canuto2007spectral}. The solver numerically solves the divergence-free Navier-Stokes equation: 
\begin{align}
    \begin{cases}
        \nabla\cdot \vu=0\\
        \partial_t \vu + \ve\nabla\vu = -\nabla \vp + \vv\Delta \vu+f
    \end{cases},
\end{align}
where $\vu$ stands for the velocity field, $\vp$ stands for the pressure, and $f$ describes an external forcing term. 
Setting the Reynolds number to 30, the spatial field has resolution $80\times 80$. We simulate the system forward for 180 seconds with $6,000$ available frames. We standardize the data within the range of $(0,1)$ and randomly fix 10 sensors from the 6,400 available spatial locations (0.16\%). 
The lag parameter is set to 360. 

For the setting of SINDy-SHRED, we slightly change the neural network setting because the output domain is 2D. 
Therefore, after the GRU unit, we use two shallow decoders to predict the output of the 2D field. The two decoders are two-layer ReLU networks with 350 and 400 neurons. 
We set the latent dimension to 3. 
The time step for forward integration is set to $dt=0.003$ which corresponds to the FPS during data generation. 
We set the batch size to 256 and the learning rate to $5e^{-4}$ using the Lion optimizer~\citep{chen2024symbolic}. The thresholding procedure is executed every 100 epoch with the total number of training epochs as $200$. The thresholds range from $(0.4, 4)$. 

As a chaotic system, the latent space of the Kolmogorov flow is much more complex than all the prior examples we considered. Thus, we further apply seasonal-trend decomposition from the original latent space. We define the representation of the latent hidden state space after decomposition as $(z_1, z_2, z_3, z_4, z_5, z_6)$, where $(z_{2i},z_{2i+1})$ is the seasonal trend pair of the original latent space.  

\begin{align}
\label{eqn:kol_equation}
\begin{cases}
    \dot{z}_1 & = -0.007 z_3 + 0.009 z_5, \\
\dot{z}_2 & = -0.207z_4, \\
\dot{z}_3 & = -0.011z_1-0.008z_5, \\
\dot{z}_4 & = 0.103z_2, \\
\dot{z}_5 & = -0.012z_1+0.006z_3. \\
\dot{z}_6 & = 0.151z_1z_2. 
\end{cases}
\end{align}

\begin{figure}[H]
    \centering
    \includegraphics[width=\textwidth]{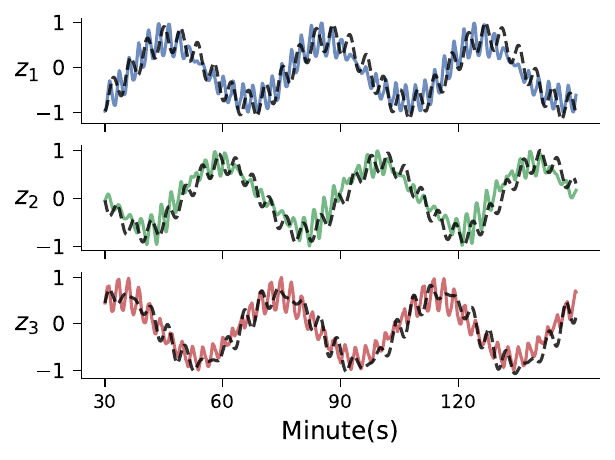}
    \caption{Extrapolation of latent representation in SINDy-SHRED from the discovered dynamical system for the 2D Kolmogorov flow data. Colored: true latent representation. Black: SINDy extrapolation.}
    \label{fig:kol_latent_space}
\end{figure}

In Eqn.~\ref{eqn:kol_equation}, we find that $z_1, z_3, z_5$ are essentially a linear system. 
$z_2, z_4, z_6$ capture higher-order effects that are difficult to model without signal separation.
We generate the trajectory from the initial condition at time point 0 and perform forward integration in Fig.~\ref{fig:kol_latent_space}. 
As we increase the Reynolds number, the discovery fails to produce robust predictions. 

This representation also demonstrates nice predictions for future frames. In Fig.~\ref{fig:kol_reconstruction}, the future prediction has an averaged MSE error of $0.035$ for all available data samples. The sensor-level prediction in Fig.~\ref{fig:kol_sensor_predictions} further demonstrates the details of the reconstruction. 
\begin{figure}[H]
    \centering
    \includegraphics[width=0.8\textwidth]{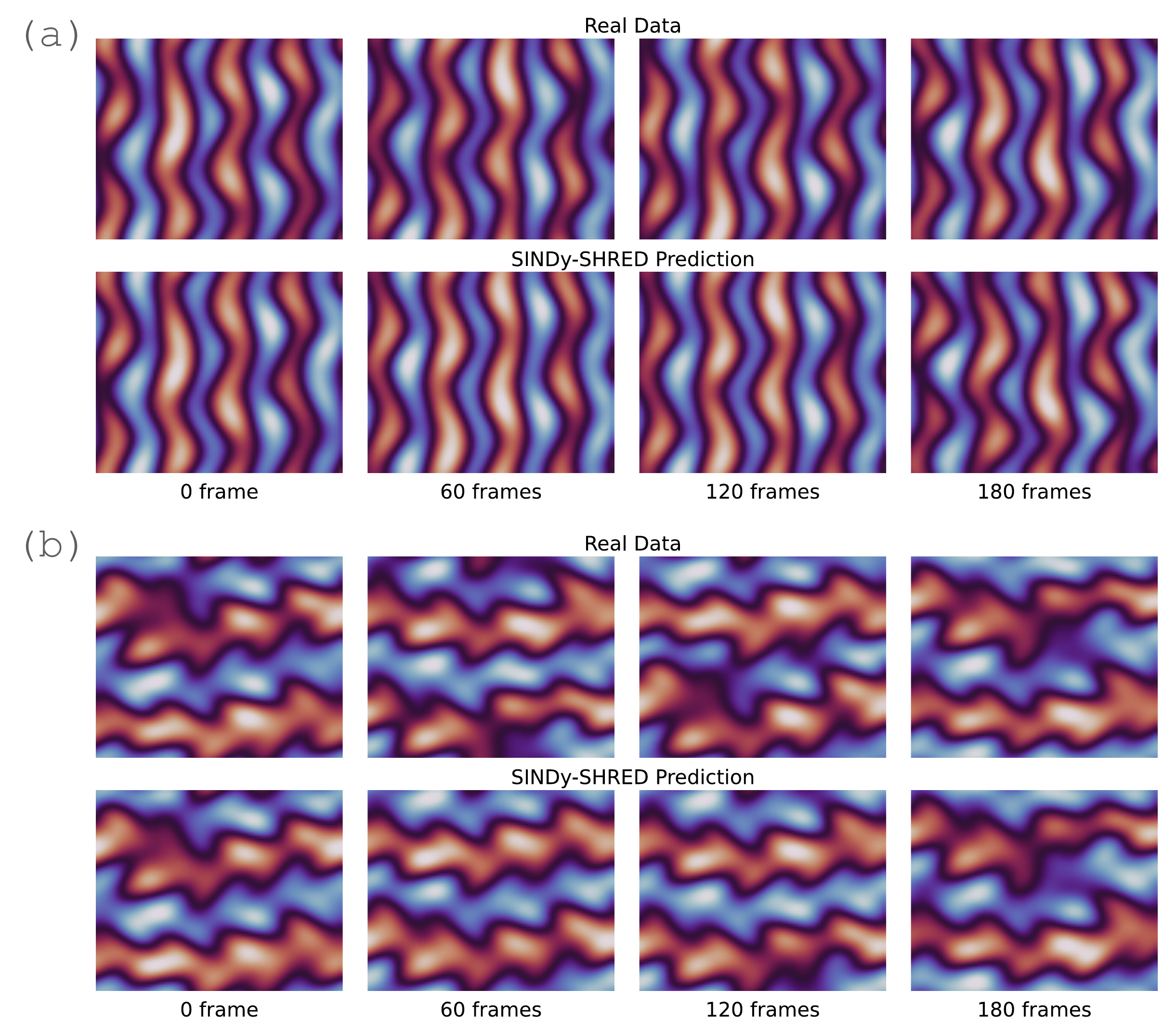}
    \caption{Long-term predictiion via SINDy-SHRED for 2D Kolmogorov flow data. }
    \label{fig:kol_reconstruction}
\end{figure}

\section{Sensor level plots of experiments}

\subsection{Sea surface temperature}

\textbf{3D visualization of SINDy-SHRED} 
\begin{figure}[H]
    \centering
    \includegraphics[width=0.8\textwidth]{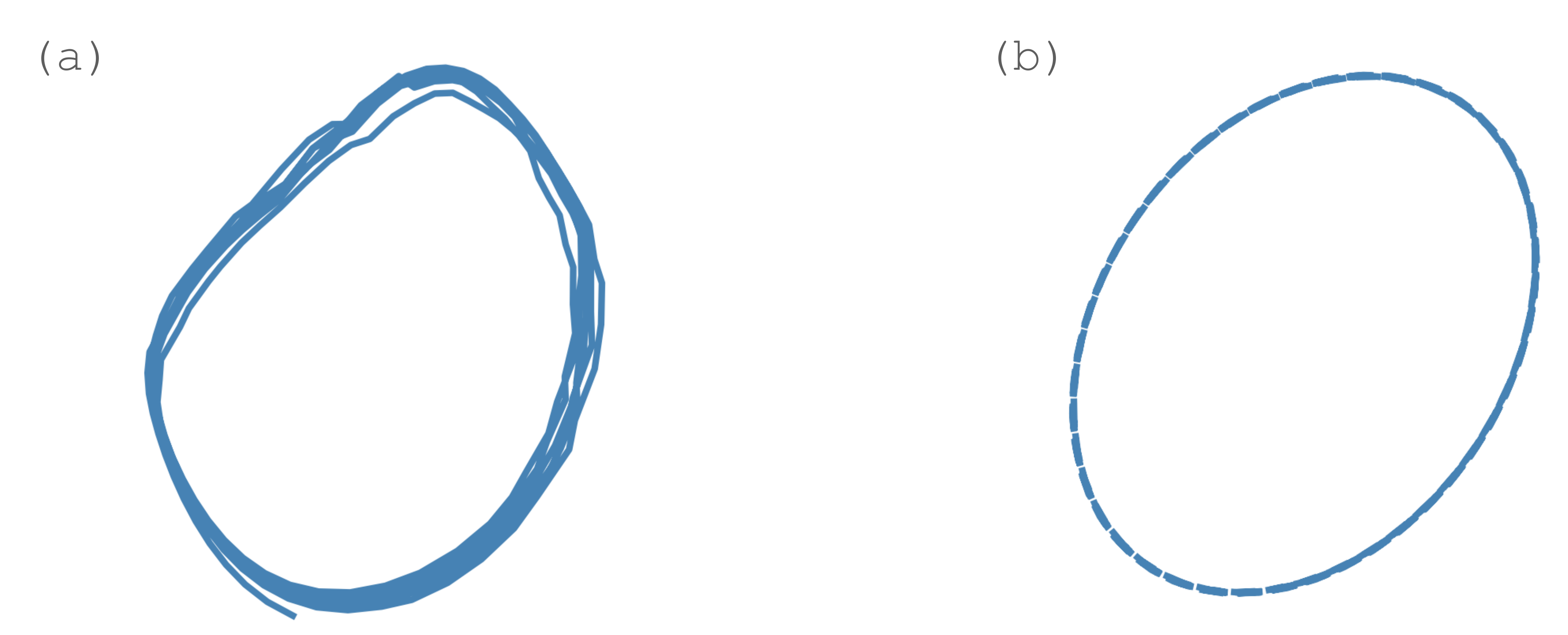}
    \caption{3D reconstruction of the original latent space and SINDy simulated latent space. }
    \label{fig:3D_sst}
\end{figure}

\textbf{Long-term extrapolation of SINDy-SHRED.} 
\begin{figure}[H]
    \centering
    \includegraphics[width=0.8\textwidth]{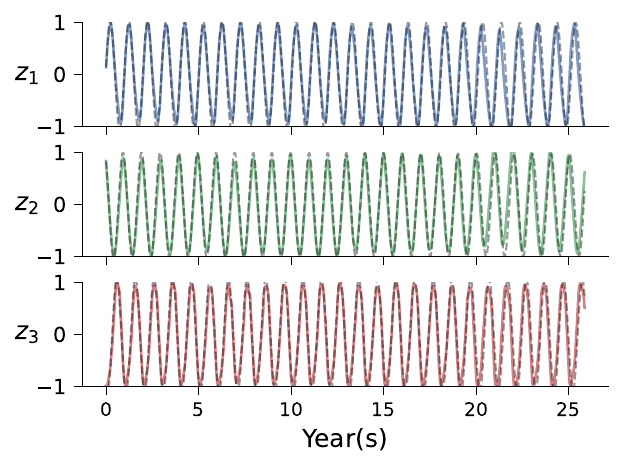}
    \caption{Extrapolation of latent representation in SINDy-SHRED from the discovered dynamical system for SST over the entire 27 years. Colored: true latent representation. Grey: SINDy extrapolation.}
    \label{fig:sst_latent_space_long_term}
\end{figure}

\textbf{Sensor-level prediction on the SST dataset.} 

\begin{figure}[H]
    \centering
    \includegraphics[width=\textwidth]{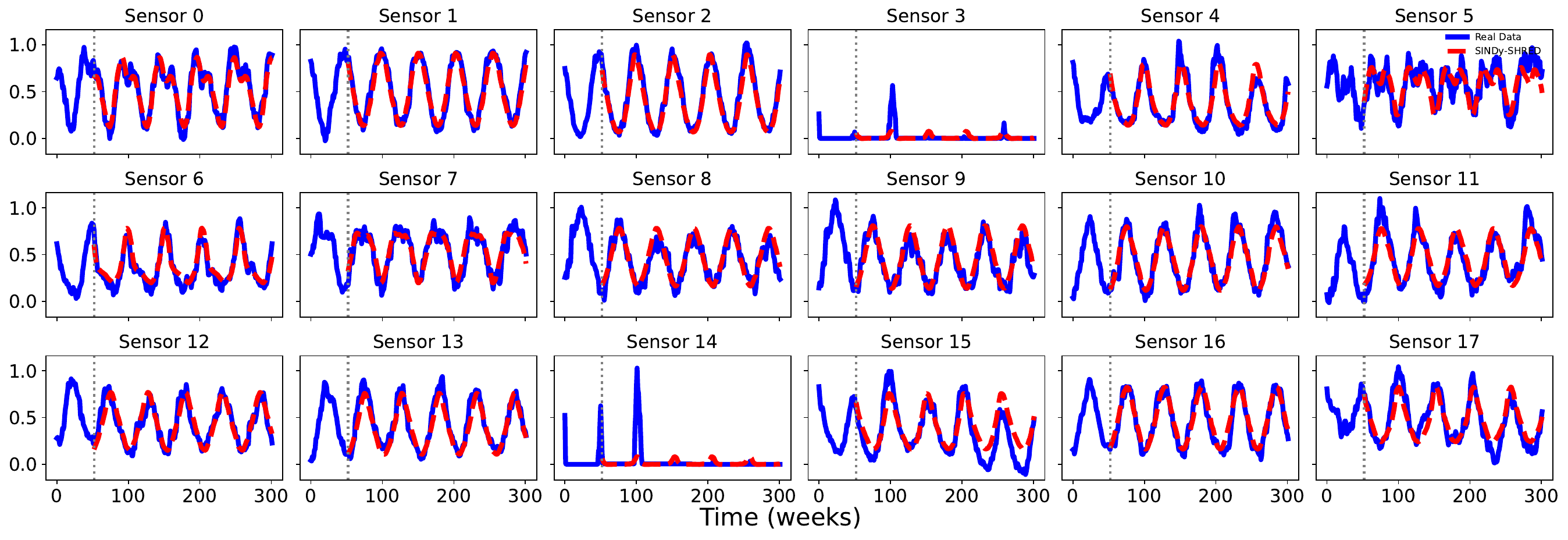}
    \caption{Extrapolation of SINDy-SHRED for sensor-level predictions on the SST data. We randomly picked 18 sensors from spatial locations that are not in the sparse sensor training. The extrapolation shows the SINDy-SHRED prediction for the following 300 weeks. }
    \label{fig:sst_sensors}
\end{figure}

\subsection{Ozone data}

\textbf{Convergence behavior of SINDy-SHRED on the Ozone dataset.}

\begin{figure}[H]
    \centering
    \includegraphics[width=0.8\linewidth]{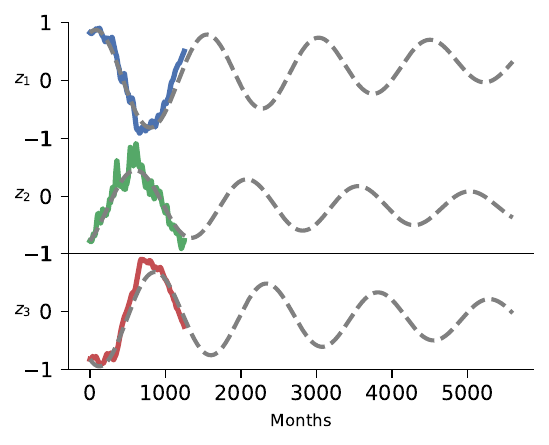}
    \caption{Long term extrapolation of Ozone data. The latent SINDy model presents a convergence behavior towards the mean-field solution. }
    \label{fig:ozone_latent_long_term}
\end{figure}

\textbf{Sensor-level prediction on the Ozone dataset.} 

\begin{figure}[H]
    \centering
    \includegraphics[width=\textwidth]{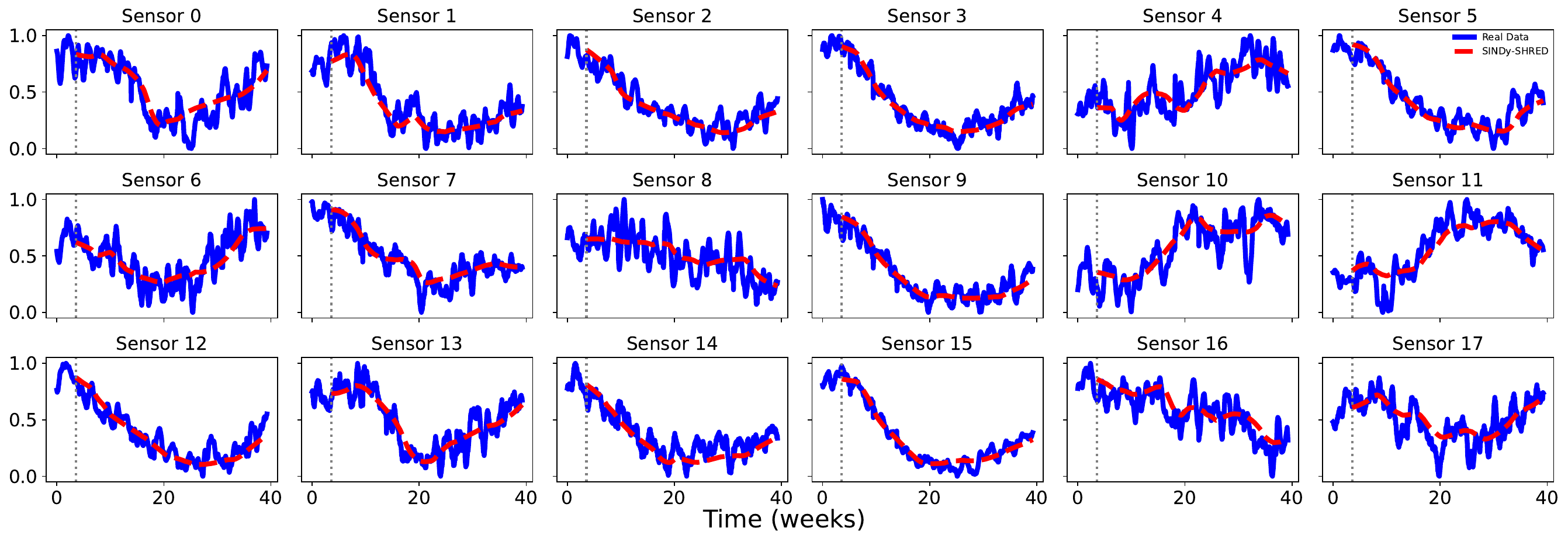}
    \caption{Extrapolation of SINDy-SHRED for sensor-level predictions on the Ozone data. We randomly picked 18 sensors from spatial locations that are not in the sparse sensor training. The extrapolation shows the SINDy-SHRED prediction for the following 40 weeks. }
    \label{fig:ozone_sensors}
\end{figure}

\newpage
\begin{figure}[H]
    \centering
    \includegraphics[width=0.8\textwidth]{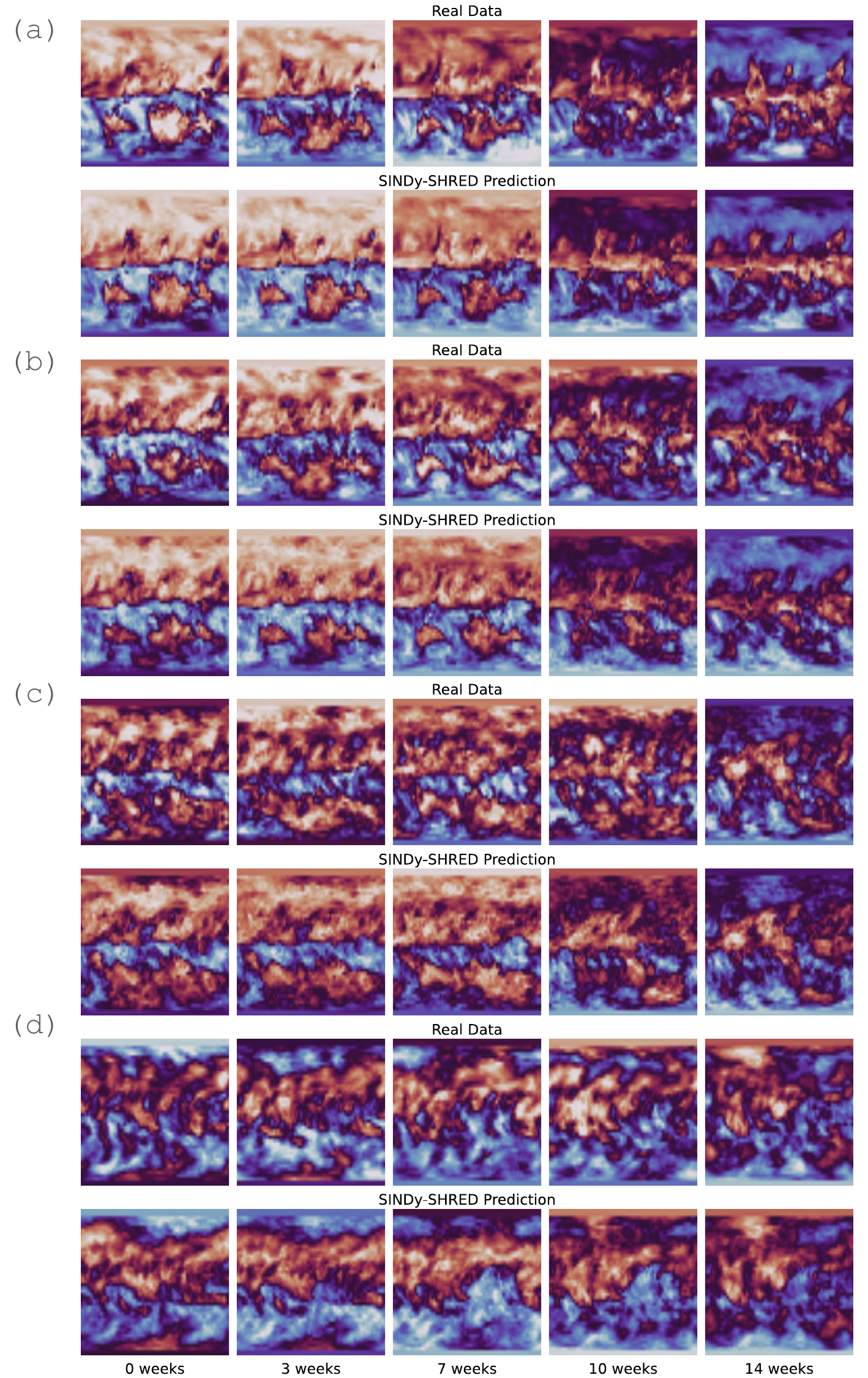}
    \caption{Reconstruction of atmospheric ozone concentration data for different elevation (a) 0 km (b) 4 km (c) 8 km (d) 12 km.   }
    \label{fig:ozone_reconstruction_all_h}
\end{figure}

\subsection{Flow over a cylinder}

\textbf{Latent space of Koopman-SHRED}

\begin{figure}[H]
    \centering
    \includegraphics[width=1.0\textwidth]{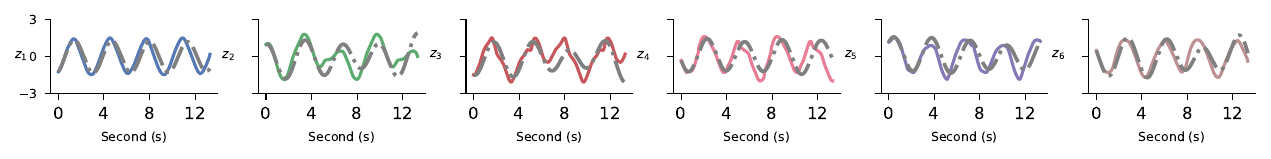}
    \caption{Extrapolation of latent representation in Koopman-SHRED from the discovered dynamical system for flow over a cylinder data. Colored: true latent representation. Grey: SINDy extrapolation. }
    \label{fig:flow_latent_space_koopman}
\end{figure}

\textbf{Sensor-level prediction on the flow over a cylinder dataset.}

\begin{figure}[H]
    \centering
    \includegraphics[width=\textwidth]{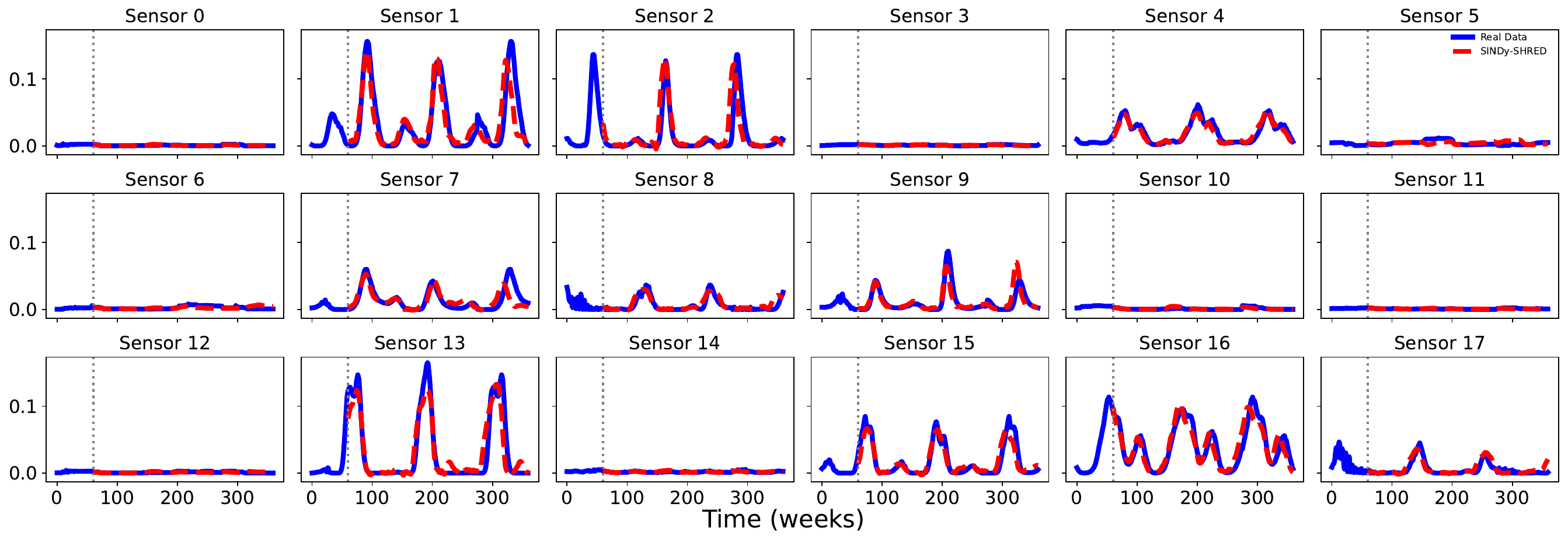}
    \caption{Extrapolation of SINDy-SHRED for sensor-level predictions on the flow over a cylinder data. We randomly picked 18 sensors from spatial locations that are not in the sparse sensor training. The extrapolation shows the SINDy-SHRED prediction for the following 400 frames. }
    \label{fig:flow_sensor_predictions}
\end{figure}

\textbf{Long-term extrapolation on the flow over a cylinder dataset.}

\begin{figure}[H]
    \centering
    \includegraphics[width=\textwidth]{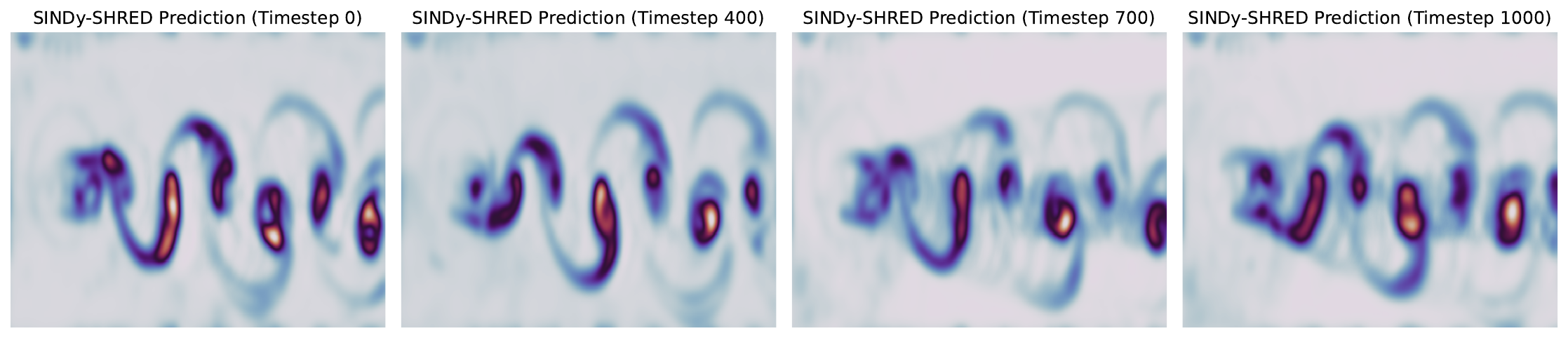}
    \caption{Prediction of the flow over a cylinder data from time step 0 (reconstruction) to 1000 frames. We note this extrapolation is completely out of the dataset. The real data for testing is only available up to 500 frames. }
    \label{fig:flow_prediction_long}
\end{figure}

\subsection{Isotropic turbulent flow}

\label{app:iso_vis}
\begin{figure}
    \centering
    \includegraphics[width=\linewidth]{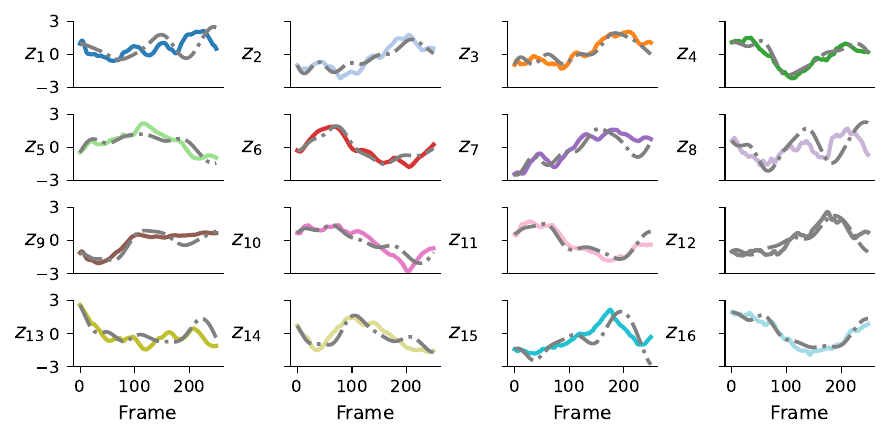}
    \caption{Extrapolation of all 16 latent representation in SINDy-SHRED from the discovered dynamical system for isotropic turbulent flow data for 250 frames. Colored: true latent representation. Grey: SINDy extrapolation.}
    \label{fig:iso_latent_space_all}
\end{figure}

\textbf{Sensor-level prediction on the moving pendulum dataset.} 

\begin{figure}[H]
    \centering
    \includegraphics[width=\textwidth]{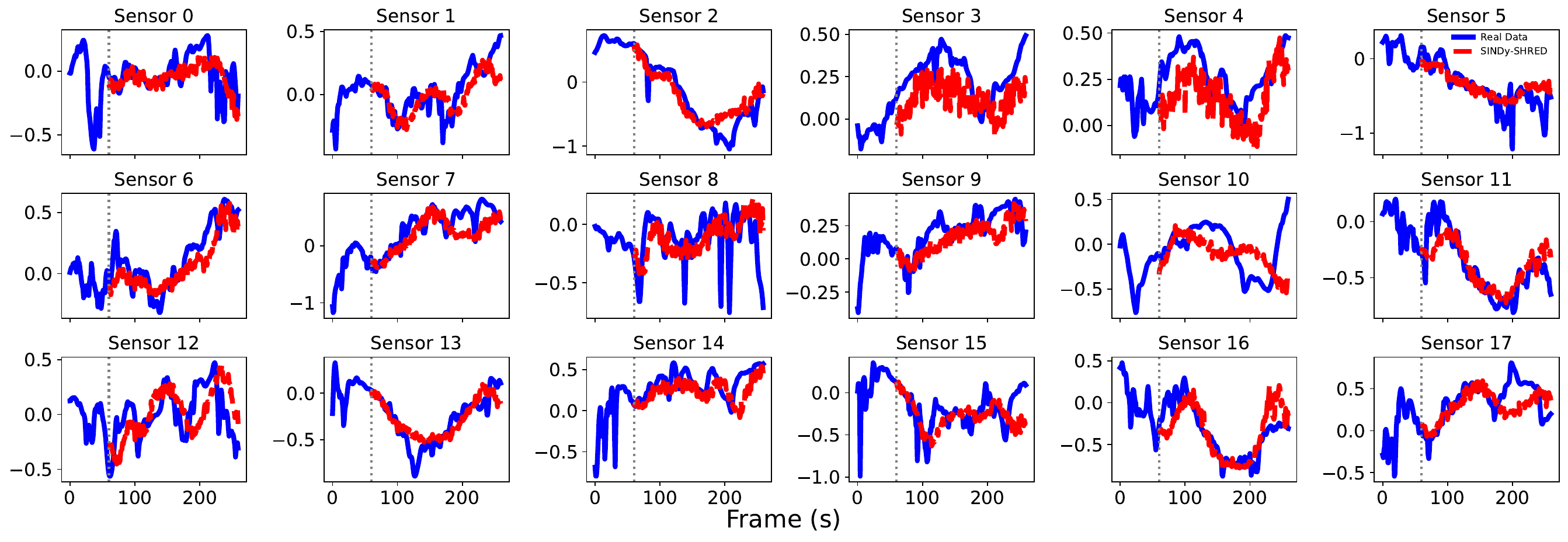}
    \caption{Extrapolation of SINDy-SHRED for sensor-level predictions on the isotropic turbulent flow data. We randomly picked 18 sensors from spatial locations that are not in the sparse sensor training. The extrapolation shows the SINDy-SHRED prediction for the following 250 frames. }
    \label{fig:iso_sensor_predictions}
\end{figure}

\newpage
\subsection{Pendulum}

\textbf{Sensor-level prediction on the moving pendulum dataset.} 

\begin{figure}[H]
    \centering
    \includegraphics[width=\textwidth]{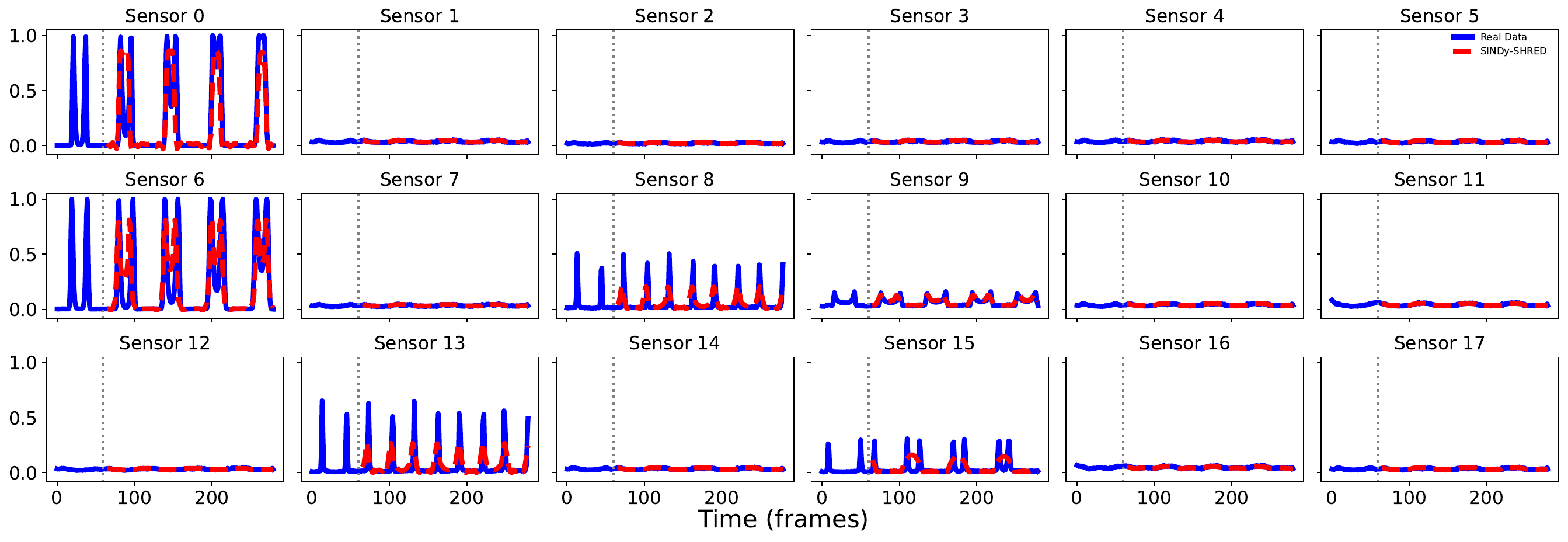}
    \caption{Extrapolation of SINDy-SHRED for sensor-level predictions on the moving pendulum data. We randomly picked 18 sensors from spatial locations that are not in the sparse sensor training. The extrapolation shows the SINDy-SHRED prediction for the following 382 frames. }
    \label{fig:pen_sensor_predictions}
\end{figure}

\subsection{Kolmogorov flow}

\textbf{Sensor-level prediction on the chaotic 2D Kolmogorov flow dataset.} 

\begin{figure}[H]
    \centering
    \includegraphics[width=\textwidth]{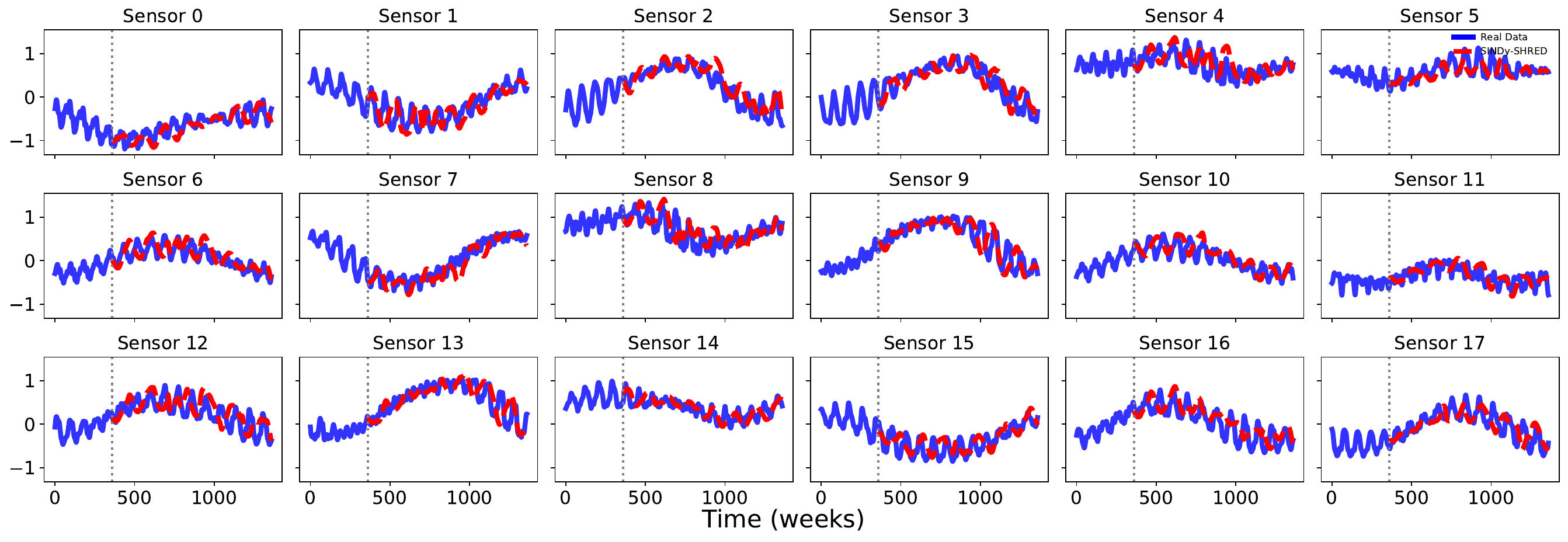}
    \caption{Extrapolation of SINDy-SHRED for sensor-level predictions on the 2D Kolmogorov flow data. We randomly picked 18 sensors from spatial locations that are not in the sparse sensor training. The extrapolation shows the SINDy-SHRED prediction for the following 1500 frames. }
    \label{fig:kol_sensor_predictions}
\end{figure}

\section{Analysis of learned ODEs}

\subsection{Sea-surface temperature}
\label{sec:analytic_sst}
The analytic solution to \ref{eqn:sst_equation}
\begin{equation}
    \mathbf z (t) = c_1 \mathbf v_1 e^{(-0.01+6.24i)t} + c_2 \mathbf{v_2}e^{(-0.01-6.24i)t} + c_3 \mathbf v_3 e^{0.02t},
\end{equation}
where $v_1 = \begin{pmatrix}
     -0.39 + 0.38i  \\
  -0.19 - 0.52i  \\
   0.63
\end{pmatrix}$, $v_2 = \begin{pmatrix}
     -0.39 - 0.38i  \\
  -0.19 + 0.52i  \\
   0.63
\end{pmatrix},$ $v_3 = \begin{pmatrix}
    0.68\\
    0.33\\
    0.65
\end{pmatrix},$ and $c_1, c_2$, and $c_3$ depend on the initial condition.
\subsection{Ozone data}
\label{sec:analytic_ozone}
The complete closed form solution to Eq. \ref{eqn:ozone_equation} is given by:
\begin{align}
    \mathbf z(t) &= T \text{diag}\begin{pmatrix}
        e^{(-0.003 + 0.0079i)t}\\
        e^{(-0.003 + 0.0079i)t}\\
        -0.003
    \end{pmatrix}T^{-1}\mathbf z_0 
    &+ T \left( \text{diag} \begin{pmatrix}
        (-42-111i)e^{(-0.003+0.0079i)t}\\
        (-42+111i)e^{(-0.003-0.0079i)t} \\
        -333 e^{-0.003t}
    \end{pmatrix}\right|_{0}^tT^{-1}\begin{pmatrix}
        -0.002 \\
        0\\
        0.002
    \end{pmatrix},
\end{align}
where $\mathbf z_0$ is the state at $t=0$ and \begin{equation}
    T = \begin{pmatrix}
   0.66 & 0.66 & -0.99\\
  -0.30 - 0.35i & -0.30 + 0.35i & 0.061 \\
  -0.59 + 0.11i & -0.59 - 0.11i & 0.16
    \end{pmatrix}. \notag
\end{equation}

\subsection{Flow over a cylinder}
\label{sec:analytic_flow}
The complete closed form solution to Eq. \ref{eqn:flow_equation_koopman} is given by: 
\begin{equation}
    \mathbf z(t) = T \text{diag}\begin{pmatrix}
        e^{(-0.01+1.52i)t} \\
        e^{(-0.01-1.52i)t} \\
        e^{(0.11+1.05i)t} \\
        e^{(0.11-1.05i)t} \\
        e^{-0.20t} \\ 
        0
    \end{pmatrix} T^{-1} \mathbf z_0
\end{equation}
where $\mathbf z_0$ is the state at $t=0$ and 
\begin{equation}
    T = \begin{pmatrix}
  -0.35 + 0.20i & -0.35 - 0.20i & 0.028 - 0.006i & 0.028 + 0.006i &  0.005 & 0 \\
   0.15 - 0.39i &  0.15 + 0.39i & -0.12 - 0.45i &  -0.12 + 0.45i &  0.68 & 0 \\
  -0.10 + 0.37i & -0.10 - 0.37i &  0.034 + 0.57i &  0.034 + 0.57i & 0.23 &  1 \\
   0.24 + 0.33i &  0.24 - 0.33i & 0.013 - 0.21i  & 0.013 + 0.21i & -0.37 &  0 \\
  -0.13 - 0.38i & -0.13 + 0.38i &  0.034 - 0.27i &  0.034 + 0.27i & -0.55 &  0 \\
   0.44 & 0.44  &  0.58 &  0.58 & 0.21 & 0
    \end{pmatrix}. \notag
\end{equation}

\subsection{Isotropic turbulence flow}
The complete closed form solution to Eq.~\eqref{eqn:iso_equation} is given by: 
\begin{equation}
    \label{eqn:iso_equation_full}
    \mathbf{z}(t) = T \text{diag}\begin{pmatrix}
        e^{(0.47+9.39i)t} \\
        e^{(0.47+9.39i)t} \\
        e^{(0.05+11.90i)t} \\
        e^{(0.05+11.90i)t} \\
        e^{(0.03+13.42i)t} \\
        e^{(0.03+13.42i)t} \\
        e^{(-0.04+3.46i)t} \\
        e^{(-0.04+3.46i)t} \\
        e^{-0.26t} \\
        e^{(-0.27+5.44i)t} \\
        e^{(-0.27+5.44i)t} \\
        e^{(-0.75+8.30i)t} \\
        e^{(-0.75+8.30i)t} \\
        e^{(-1.37+18.72i)t} \\
        e^{(-1.37+18.72i)t} \\
        e^{-3.39t} \\
    \end{pmatrix} T^{-1} \mathbf{z}_0
\end{equation}
where $\mathbf{z}_0$ is the state at $t=0$ and $T$ contains all eigenvectors.

\end{document}